\documentclass[11pt]{article}
\usepackage[utf8]{inputenc}
\usepackage{amsmath, amssymb}
\usepackage{amsthm}
\usepackage{algorithm}
\usepackage{fancyhdr}
\usepackage{booktabs} 
\usepackage{algpseudocode} 
\newtheorem{proposition}{Proposition}
\usepackage{breqn}
\allowdisplaybreaks
\newtheorem{definition}{Definition}
\newtheorem{theorem}{Theorem}
\usepackage[final]{acl}
\usepackage{svg}
\usepackage{times}
\usepackage{latexsym}
\usepackage[T1]{fontenc}

\usepackage[utf8]{inputenc}

\usepackage{microtype}

\usepackage{inconsolata}

\usepackage{graphicx}
\usepackage{subcaption}
\usepackage{placeins}

\title{Calibration-Aware Policy Optimization for Reasoning LLMs}

\author{
  \textbf{Ziqi Wang\textsuperscript{1,2}},
  \textbf{Xingzhou Lou\textsuperscript{1,2}},
  \textbf{Meiqi Wu\textsuperscript{1}},
  \textbf{Zhengqi Wen\textsuperscript{3}},
\\
  \textbf{Junge Zhang\textsuperscript{1,2}}
\\
  \textsuperscript{1}National Key Laboratory of Cognition and Decision Intelligence for Complex Systems,
  \\
  Institution of Automation, Chinese Academy of Sciences
\\
  \textsuperscript{2}School of Artificial Intelligence, University of Chinese Academy of Sciences
\\
  \textsuperscript{3}Beijing National Research Center for Information Science and Technology, Tsinghua University
\\
  \small{
    \textbf{Correspondence to Junge Zhang:} \href{mailto:email@domain}{jgzhang@nlpr.ia.ac.cn}
  }
}
\begin{document}
\maketitle
\thispagestyle{fancy}
\begin{abstract}
Group Relative Policy Optimization (GRPO) enhances LLM reasoning but often induces overconfidence, where incorrect responses yield lower perplexity than correct ones, degrading relative calibration as described by the Area Under the Curve (AUC). Existing approaches either yield limited improvements in calibration or sacrifice gains in reasoning accuracy. 
We first prove that this degradation in GRPO-style algorithms stems from their uncertainty-agnostic advantage estimation, which inevitably misaligns optimization gradients with calibration. This leads to improved accuracy at the expense of degraded calibration. We then propose \textbf{C}alibration-\textbf{A}ware \textbf{P}olicy \textbf{O}ptimization \textbf{(CAPO)}. It adopts a logistic AUC surrogate loss that is theoretically consistent and admits regret bound, enabling uncertainty-aware advantage estimation. By further incorporating a noise masking mechanism, CAPO achieves stable learning dynamics that jointly optimize calibration and accuracy.
Experiments on multiple mathematical reasoning benchmarks show that CAPO-1.5B significantly improves calibration by up to 15\% while achieving accuracy comparable to or better than GRPO, and further boosts accuracy on downstream inference-time scaling tasks by up to 5\%. Moreover, when allowed to abstain under low-confidence conditions, CAPO achieves a Pareto-optimal precision–coverage trade-off, highlighting its practical value for hallucination mitigation.
\end{abstract}

\section{Introduction}
Model calibration is defined as the correlation between a model’s confidence in its answers and the ground truth correctness of those answers \citep{geng2024survey}. From the era of conventional neural networks to the advent of Large Language Models (LLMs), calibration has consistently been a key focus due to its paramount importance in two aspects \citep{xiao2025restoring, tao2024trust, kadavath2022language}. 

First, calibration is fundamental to model trustworthiness, enabling reliable uncertainty estimation and abstention to mitigate hallucination—the generation of plausible but factually incorrect assertions. This is particularly vital in high-stakes domains such as finance and healthcare \citep{savage2025large}.
Second, confidence estimates are widely used to guide algorithmic decisions.
Multi-agent or cascading systems strengthen model collaboration when a single model is uncertain \citep{luo2025llmpoweredmultiagentautomatedcrypto,warren2025bidirectionalmodelcascadingproxy,chuang2025learningroutellmsconfidence}. Self-paced training prioritizes uncertain samples \citep{feng2025your,wang2025entropybasedadaptiveweightingselftraining}. And inference-time scaling strategies select candidate responses with high confidence \citep{stoisser2025towards,vashurin2025uncertainty,zhou2025bridging}.
All these approaches critically depend on well-calibrated confidence estimates to reflect the true correctness of model outputs. 

However, several studies have indicated that Reinforcement Learning from Verifiable Rewards (RLVR) (\citealp{shao2024deepseekmath};\citealp{yu2025dapo};\citealp{zheng2025groupsequencepolicyoptimization}) can make models overconfident (\citealp{liu2025c};\citealp{kalai2025language};\citealp{bereket2025uncalibrated}). This often manifests as degradation in relative calibration, where the model outputs fluent but incorrect answers with higher perplexity than correct ones. 

Existing efforts to address the overconfidence induced by algorithms such as GRPO \citep{shao2024deepseekmath} and GSPO \citep{zheng2025groupsequencepolicyoptimization} primarily rely on heuristic designs. These works often lack quantitative analyses of calibration and, critically, provide no theoretical guarantees for calibration improvement. As a result, they either yield only limited calibration gains, as in CoDaPO \citep{zhou2025codapo} and CDE \citep{dai2025cde}, or compromise the model’s overall accuracy, as in SimKO \citep{peng2025simko}.

To address these issues, we first experimentally show that during the GRPO optimization process, the typical relative calibration metric AUC \citep{ling2003auc,vashurin2025uncertainty,zhou2025bridging} progressively worsens as accuracy improves. GRPO-style methods construct advantages from group-wise reward differences. We prove that reward-only sample evaluation induces optimization gradients corresponding to an inconsistent surrogate for calibration optimization \citep{gao2012consistency}. Consequently, the learning process is not aligned with calibration improvement, which explains the degradation in experiments.

Building upon this discovery, we propose a novel approach \textbf{C}alibration-\textbf{A}ware \textbf{P}olicy \textbf{O}ptimization \textbf{(CAPO)}, which (1) adopts uncertainty-aware advantage estimation based on consistent logistic AUC surrogate, enabling joint optimization of calibration and accuracy, and (2) introduces a reference-model-based noise masking mechanism to maintain training stability. In addition to this theoretical justification, a complementary gradient analysis further illustrates how the learning dynamics naturally prioritizes the correction of responses with misaligned confidence, offering additional intuition.

Experimental results on Qwen2.5-Math-1.5B and 7B models demonstrate that CAPO enables stable and joint optimization of both calibration and accuracy. It successfully prevents calibration degradation and attains accuracy that matches or surpasses that of GRPO, as shown in Figure~\ref{parento-auc-grpo-gspo}. Our contributions are summarized as follows:

\begin{itemize}
    \item We provide \textbf{experiment evidence and theoretical explanation} of calibration degradation in GRPO-style algorithms.
    \item To address this issue, we introduce a \textbf{theoretically grounded and consistent optimization objective} that enables joint optimization of accuracy and calibration.
    \item \textbf{Extensive experiments across six benchmarks} demonstrate significant calibration improvements over GRPO, GSPO, and three additional baselines with preserved accuracy, yielding Pareto-optimal precision–coverage trade-offs for hallucination mitigation and improved inference-time scaling accuracy.
\end{itemize}

\begin{figure*}[t]
  \centering
  \begin{subfigure}[t]{0.32\textwidth}
    \centering
    \includegraphics[width=\linewidth]{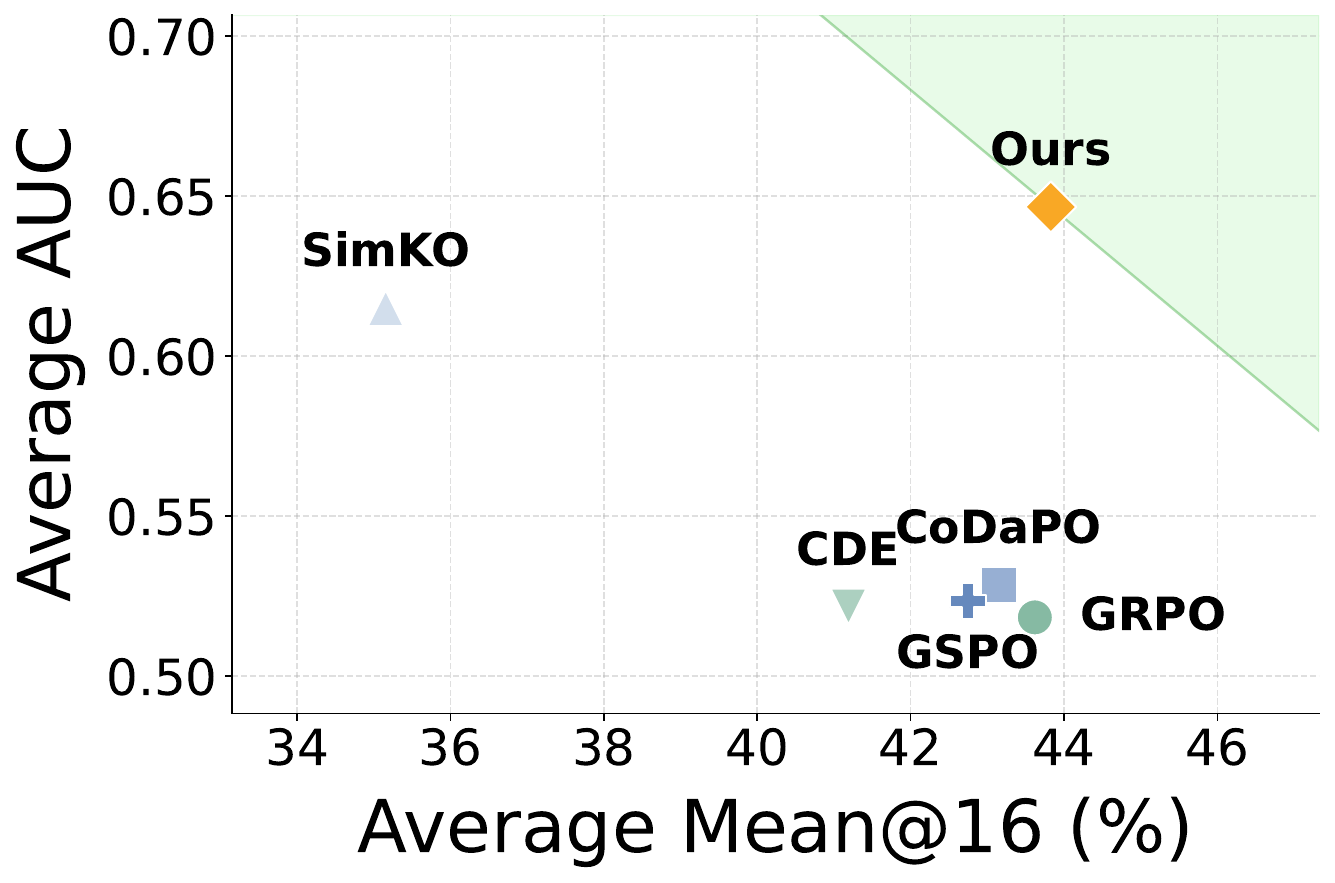}
    \caption{}
  \end{subfigure}\hfill
  \begin{subfigure}[t]{0.32\textwidth}
    \centering
    \includegraphics[width=\linewidth]{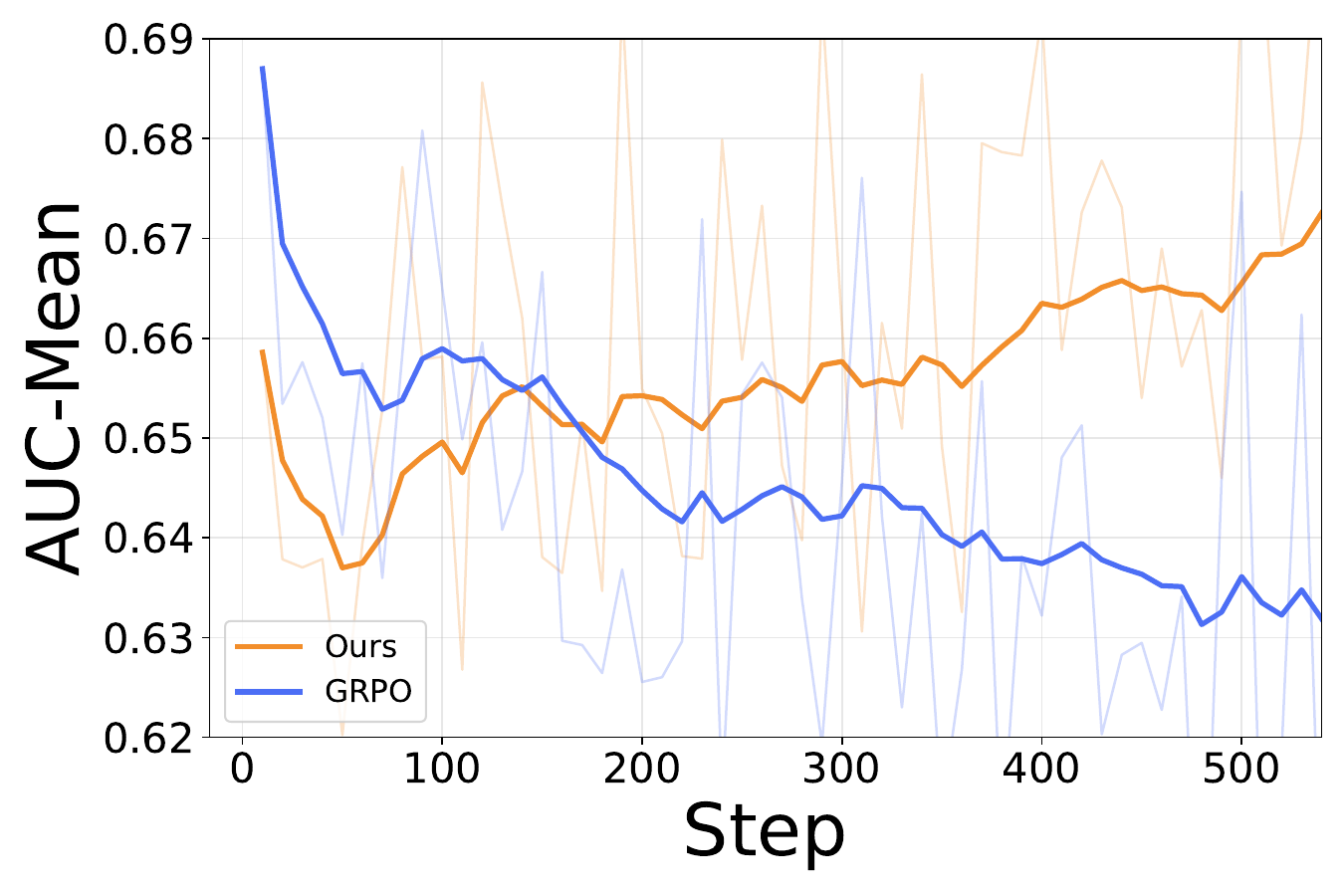}
    \caption{}
  \end{subfigure}\hfill
  \begin{subfigure}[t]{0.32\textwidth}
    \centering
    \includegraphics[width=\linewidth]{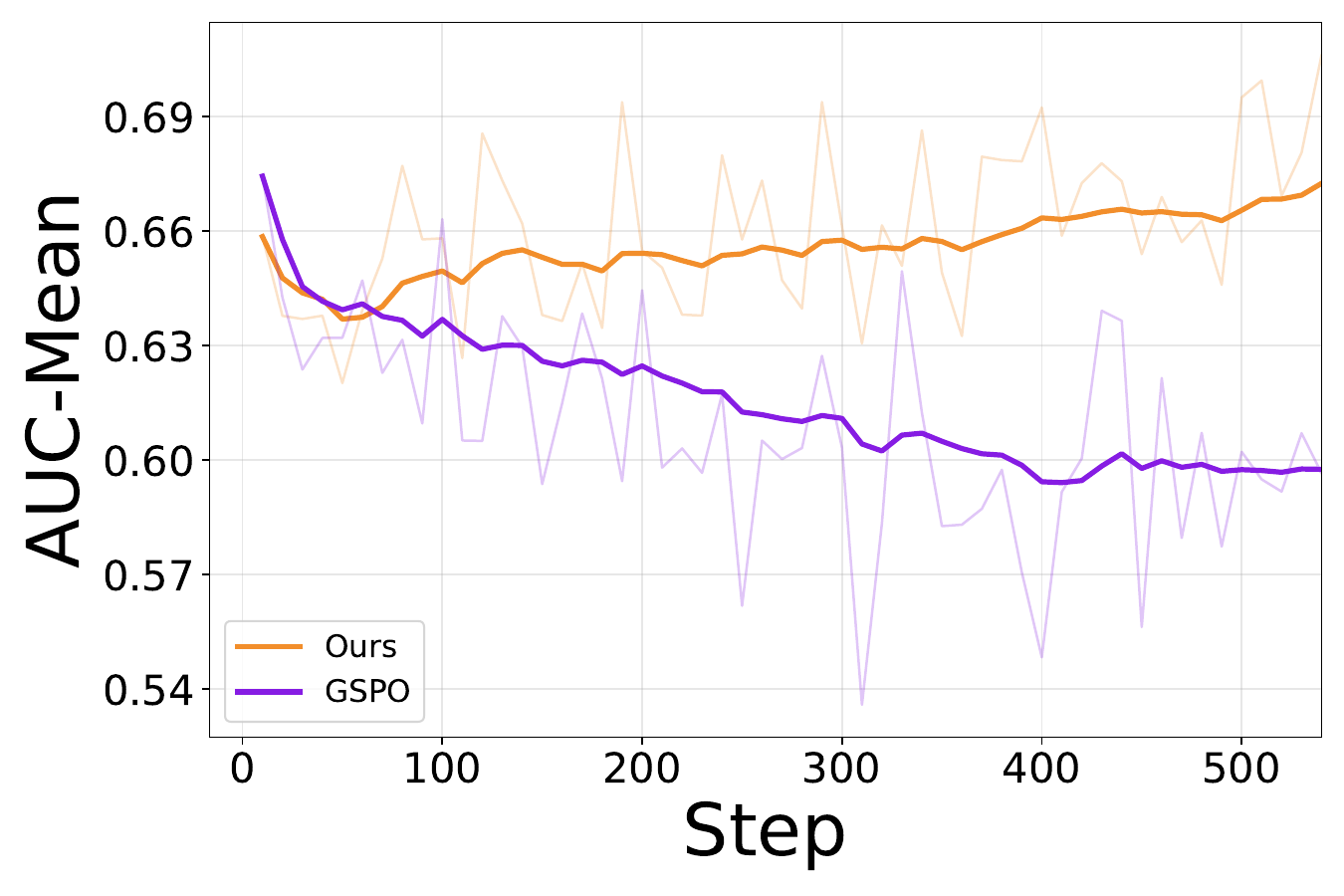}
    \caption{}
  \end{subfigure}
  \caption{(a) Comparison of average calibration (measured by AUC-mean) and accuracy (measured by mean@16) across six test benchmarks for our method and all baselines on Qwen2.5-Math-7B; (b)(c) Comparison of calibration dynamics on the validation set during training between our method and GRPO (left) / GSPO (right).}
  \label{parento-auc-grpo-gspo}
\end{figure*}

\section{Preliminary and Related Work}
As our approach builds upon AUC optimization techniques to analyze and mitigate relative calibration degradation in GRPO-style algorithms, this section reviews relevant background on GRPO and calibration, along with related prior work.
\subsection{Group Relative Policy Optimization}

The reasoning ability of LLMs can be optimized using outcome reward-based Reinforcement Learning, by modeling the generation process as an Markov Decision Process. GRPO achieves this efficiently by computing the advantage value using the reward differences among samples within a group, thereby obviating the need for a value model. For a specific question-answer pair $(q, a)$, the behavior policy $\pi_{\theta_{\text{old}}}$ samples a group of $G$ individual responses $\{ o_i \}_{i=1}^{G}$. Then, the advantage $\hat{A}_{i,t}$ of the $i$-th response is calculated by normalizing the group-level rewards$\{ R_i \}_{i=1}^{G}$:
\begin{equation}
    \hat{A}_{i,t} = 
    \frac{ R_i - \mathrm{mean}(\{ R_i \}_{i=1}^{G}) }
         { \mathrm{std}(\{ R_i \}_{i=1}^{G}) },
\end{equation}
 where $R = 1$ indicates a correct response and $R = 0$ indicates an incorrect one.

GRPO adopts a ppo-style clipped objective, together with a directly imposed KL penalty term:

\begin{equation}
\label{grpo-obj}
\resizebox{\columnwidth}{!}{
$
\begin{aligned}
&J_{\text{GRPO}}(\theta)=
\mathbb{E}_{(q,a)\sim \mathcal{D},\, \{o_i\}_{i=1}^{G} \sim \pi_{\theta_{\mathrm{old}}}(\cdot \mid q)}
\Bigg[
\frac{1}{G}\sum_{i=1}^G \frac{1}{|o_i|}\sum_{t=1}^{|o_i|}
\\
&\qquad \qquad
\min\Big(
r_{i,t}(\theta)\hat A_{i,t},
\mathrm{clip}\!\big(r_{i,t}(\theta),1-\epsilon,1+\epsilon\big)\hat A_{i,t}
\Big)
\\
&\qquad  \qquad 
-\beta D_{\mathrm{KL}}\!\left(\pi_\theta \,\|\, \pi_{\mathrm{ref}}\right)
\Bigg],
\end{aligned}
$}
\end{equation}

where
\begin{equation}
    r_{i,t}(\theta)
    = 
    \frac{
        \pi_\theta(o_{i,t} \mid q, o_{i,<t})
    }{
        \pi_{\theta_{\text{old}}}(o_{i,t} \mid q, o_{i,<t})
    }.
\end{equation}
The clipping operator $\mathrm{clip}(\cdot,\,1-\epsilon,\,1+\epsilon)$ is introduced
to constrain each policy update, preventing the new policy from drifting too far
away from the previous one by restricting the policy ratio within the interval
$[1-\epsilon,\,1+\epsilon]$.
\begin{equation}
\resizebox{\columnwidth}{!}{
$
\mathbb{D}_{\mathrm{KL}}\!\left( \pi_\theta \,\|\, \pi_{\mathrm{ref}} \right)
=
\frac{\pi_{\mathrm{ref}}(o_{i,t} \mid q, o_{i,<t})}
     {\pi_\theta(o_{i,t} \mid q, o_{i,<t})}
-\log\!\left(
\frac{\pi_{\mathrm{ref}}(o_{i,t} \mid q, o_{i,<t})}
     {\pi_\theta(o_{i,t} \mid q, o_{i,<t})}
\right)
-1 .$}
\label{eq:kl_estimation}
\end{equation}
In addition, the KL divergence between the current policy and the reference policy (i.e., the base model) is approximated as Equation (\ref{eq:kl_estimation}).

While base models are shown to be well-calibrated \citep{kalai2025language}, algorithms in the GRPO-like family are observed to cause model calibration collapse, where models generate confident, yet incorrect, responses (\citealp{liu2025c};\citealp{dai2025cde}). Prior RL-based efforts to mitigate this focus on heuristics: reward/advantage shaping (\citealp{dai2025cde};\citealp{zhou2025codapo}), regularization (\citealp{liu2025c}), or label smoothing (\citealp{peng2025simko}).

However, these works rely primarily on qualitative descriptions of calibration degradation. As a result, their heuristic approaches lack theoretical guarantees and generally fail to balance accuracy with calibration preservation.

Our work addresses these deficiencies from a principled quantitative perspective and offers a theoretical justification for calibration optimization, and a proper balance between performance and trustworthiness.

\subsection{Calibration Metric}
Perplexity (PPL) is a widely used uncertainty indicator in free-form generation \citep{stoisser2025towards, vashurin2025uncertainty, zhou2025bridging}:
\begin{equation}
\label{ppl}
\text{PPL}(o)=\exp\!\left(-\frac{1}{|o|}\sum_{t=1}^{|o|}\log \pi(o_t \mid o_{<t})\right).
\end{equation}
PPL reflects the model’s intrinsic uncertainty without incurring additional computational overhead. In contrast, approaches that elicit explicit confidence via prompting \citep{zeng2025thinking,wen2024mitigating,damani2025beyond} are often sensitive to prompt design \citep{yang2024verbalizedconfidencescoresllms} and self-consistency methods \citep{xiong2024llmsexpressuncertaintyempirical,manakul2023selfcheckgpt,tanneru2024quantifying} increase inference costs, making them less general. We therefore consider them outside the scope of this work.

Model calibration can generally be quantified using two families of metrics (\citealp{geng2024survey}): 
\textbf{relative calibration} (e.g., AUC) and \textbf{absolute calibration} (e.g., Expected Calibration Error).  

Relative calibration emphasizes the model’s ability to \emph{rank} samples—assigning higher confidence to correct responses than to incorrect ones—formalized as ($f$ denotes the confidence scoring function):
\begin{equation}
\begin{aligned}
f(q, {o}_i) \le f(q, {o}_j)
\Longleftrightarrow 
R_i
\le
R_j,       
\end{aligned}    
\end{equation}
A widely used metric for relative calibration, the Area Under the Curve (AUC) \citep{yang2022auc,yuan2021large,zhu2022benchmarking}, directly quantifies the probability that a model assigns higher confidence to a correct response than to an incorrect one, thereby capturing the core notion of relative calibration:
\begin{equation}
\label{auc definition}
\resizebox{\columnwidth}{!}{
$
\mathrm{AUC}(\pi,q,f) =
\mathbb{E}_{o_i,\,o_j \sim \mathcal{D}}
\big[\,\mathbb{I}\big((R_i - R_j)
(f(o_i) - f(o_j)) > 0\big)\big],
$}
\end{equation}
where $o_i$ and $o_j$ denote a randomly sampled pair of responses generated by the model for the same question $q$.  
We further define AUC-mean as the average AUC score across all questions in a dataset, serving as an overall measure of the model’s calibration performance:

\begin{equation}
\text{AUC-mean}(\pi,\mathcal{Q},f)
=
\frac{1}{|\mathcal{Q}|}
\sum_{q \in \mathcal{Q}}
\mathrm{AUC}(\pi,q,f),
\end{equation}
where $\mathcal{Q}$ denotes the set of questions.

Absolute calibration indicates that a model’s returned confidence matches its true correctness likelihood:
\begin{dmath}
P(R=1 \mid f(q, o) = k) = k.
\end{dmath}

In this work, we focus on relative calibration rather than absolute calibration for two reasons: (1) PPL is not directly comparable to correctness probability, making absolute calibration ill-defined, and (2) absolute calibration is an aggregate statistical property that does not ensure instance-level discrimination: for example, assigning identical confidence (e.g., 0.5) to both correct and incorrect responses perfectly satisfy absolute calibration criteria yet fails to distinguish them, limiting its usefulness for downstream tasks such as inference-time scaling. In contrast, relative calibration explicitly enforces pairwise separability.

\subsection{Consistency of AUC Optimization}
Maximizing the AUC is equivalent to minimizing the risk: $L(f)=-\mathrm{AUC}(\pi,q,f)$. However, this objective is non-differentiable and non-convex, thus a direct optimization typically results in NP-hard problems (\citealp{gao2012consistency}).

To address this difficulty, prior work commonly replaces the indicator in Equation (\ref{auc definition}) with a pairwise differentiable surrogate loss:
\begin{equation}
\label{eq:surrogate}
\Psi(f,o_i,o_j)
=
\phi\big[(R_i - R_j)\,(f(o_i) - f(o_j))\big],
\end{equation}
where $\phi$ is a convex function such as the exponential loss 
$\phi(t)=e^{-t}$, hinge loss $\phi(t)=\max(0,1-t)$, 
or squared loss $\phi(t)=(1-t)^2$(\citealp{gao2013one};\citealp{zhao2011online};\citealp{kotlowski2011bipartite};\citealp{calders2007efficient};\citealp{charoenphakdee2019symmetric}).

AUC consistency is defined as follows:
\begin{definition}
\label{def:auc-consistency}
The surrogate loss $\phi$ is said to be \emph{consistent with AUC} if, for every sequence $\{f^{\langle n\rangle}(x)\}_{n\ge 1}$,
\[
L_{\phi}\!\left(f^{\langle n\rangle}\right) \to L_{\phi}^{*}
\quad\Rightarrow\quad
L\!\left(f^{\langle n\rangle}\right) \to L^{*},
\]
where $L_\phi$ is the surrogate risk and $L$ is the true AUC risk. $L_{\phi}^{*}$ and $L^{*}$ is the corresponding optimal value.
\end{definition}

Intuitively, consistency emphasizes the \emph{asymptotic correctness} of a surrogate loss: convergence to the optimal surrogate risk guarantees convergence to the optimal AUC. 
For this reason, consistency is widely regarded as one of the most important theoretical properties of surrogate losses.

A sufficient condition for achieving AUC consistency is provided by the following theorem (\citealp{gao2012consistency}):
\begin{theorem}
\label{thm:consistency}
The surrogate loss 
$\Psi(f,o_i,o_j)$ 
is consistent with AUC if $\phi:\mathbb{R}\rightarrow\mathbb{R}$ is a convex, differentiable, and non-increasing function with $\phi'(0) < 0$.
\end{theorem}
It can be readily shown that surrogate losses such as logistic loss and exponential loss are consistent with AUC according to the above theorem.

Furthermore, under the realizable setting where each response is assigned a correct or incorrect label with probability one, 
\citet{gao2012consistency} derives the following regret bound, which theoretically guarantees that 
minimizing the specified surrogate risk is aligned with optimizing AUC:
\begin{theorem}
\label{regret bound}
For exponential loss, hinge loss, general hinge loss, $q$-norm hinge loss, and least square loss, we have
\begin{equation}
L(f) - L^* \le L_\phi(f) - L_\phi^*,
\end{equation}
and for logistic loss, we have
\begin{equation}
L(f) - L^* \le \frac{1}{\ln 2}\bigl(L_\phi(f) - L_\phi^*\bigr).
\end{equation}
\end{theorem}

Prior work on AUC optimization has primarily focused on training better classifiers on static datasets, whereas we study the relationship between reinforcement learning dynamics and relative calibration, and leverage AUC optimization techniques to achieve simultaneous improvements in both calibration and accuracy. 

\section{Why GRPO-style Objectives Degrade Calibration}
\paragraph{Empirical Observation}
Prior works (\citealp{yu2025dapo}) have shown that the KL regularization term in GRPO is detrimental to optimization, while other studies (\citealp{liu2025understanding}) indicate that removing the standard deviation term in advantage estimation yields an unbiased estimator. Accordingly, we adopt these optimized configurations when implementing GRPO. Empirically, as shown in Figure~\ref{parento-auc-grpo-gspo}(b), during GRPO training, the test-set accuracy consistently improves, whereas the AUC-mean steadily deteriorates. Below, we provide a theoretical explanation for this phenomenon.
\paragraph{Theoretical Explanation}
At its core, the GRPO gradient is a REINFORCE-style gradient with group relative advantage estimation, as expressed in Equation (\ref{grpo-grad}).
\begin{equation}
\label{grpo-grad}
\resizebox{\columnwidth}{!}{
$
\mathbb{E}_{o_{1:G}\sim \mathcal{D}}
\left[
\frac{1}{G}
\sum_{i=1}^{G}
\hat{A_i} \nabla_\theta lpm_\theta(o_i)
\right],
\hat{A_i} = R_i - \bar{R},
$}
\end{equation}
where $lpm_\theta(o_i)=\frac{1}{|o|}\sum_{t=1}^{|o|}\log \pi_\theta(o_t\mid o_{<t})$ (see Appendix~\ref{gradient} for a detailed derivation).
Consequently, this gradient can be equivalently rewritten as Equation (\ref{grpo-grad2}). Constant scaling factors are omitted for clarity in the subsequent analysis.
\begin{equation}
\label{grpo-grad2}
\resizebox{\columnwidth}{!}{
$
\mathbb{E}_{o_{1:G}\sim \mathcal{D}}
\left[
\sum_{1 \le i < j \le G}
\left(
\nabla_\theta lpm_\theta(o_i) - \nabla_\theta lpm_\theta(o_j)
\right)
\left(
R_i - R_j
\right)
\right].
$}
\end{equation}
By invoking the unbiasedness property of U-statistics (see Appendix~\ref{U} for a formal definition), namely that the expectation over a group of samples equals the expectation over a randomly sampled pair, it follows:
\begin{equation}
\label{grpo-pair-grad}
\resizebox{\columnwidth}{!}{$
\begin{alignedat}{2}
&\mathbb{E}_{o_1,o_2\sim\mathcal{D}}
\Big[
(\nabla_\theta lpm_\theta(o_1) - \nabla_\theta lpm_\theta(o_2))
(R_1 - R_2)
\Big]
\\
= \;&\nabla_\theta \mathbb{E}_{o_1,o_2\sim\mathcal{D}}
\Big[
(lpm_\theta(o_1) - lpm_\theta(o_2))
(R_1 - R_2)
\Big].
\end{alignedat}
$}
\end{equation}
Therefore, we have the following result:
\begin{theorem}
The gradient of the GRPO objective coincides with the gradient of the AUC optimization objective with surrogate loss $\phi(t) = -t$ and scoring function $f = lpm$ (i.e., perplexity). This surrogate loss is inconsistent for AUC optimization (see Appendix~\ref{proof of th3} for the proof).
\end{theorem}

That is, while optimizing the GRPO objective (which effectively optimizes accuracy), the true AUC is not theoretically guaranteed to improve simultaneously. In practice, GRPO tends to overfit easy samples, leading to a sharpened output distribution, in which the perplexities of both positive and negative samples decrease concurrently, thereby degrading AUC. This theoretical prediction aligns well with our empirical observations shown in Figure~\ref{parento-auc-grpo-gspo}(b).

The above analysis extends beyond GRPO to any algorithm relying on  reward-only advantage estimator: for instance, Figure~\ref{parento-auc-grpo-gspo}(c) demonstrates similar calibration degradation in GSPO, where although optimization is sequence-level, advantage estimation remains reward-based. Evaluating samples \textbf{purely based on reward while ignoring uncertainty} inevitably yields an AUC-inconsistent surrogate (proof in Appendix~\ref{proof of th3}), thereby training the model to pursue only high reward values rather than both honesty and accuracy.

\section{Method}
To address the issues in Section~3, we propose an uncertainty-aware advantage estimation derived from a consistent AUC surrogate. 
\subsection{Calibration-Aware Policy Optimization}
\paragraph{Consistent Surrogate Objective}
Building upon the theoretical analysis in Section~3, we propose to replace the inconsistent surrogate loss implicitly induced by GRPO with a \emph{logistic surrogate loss}: $\phi_\tau(t) = \log\big(1 + \exp(-t/\tau)\big)$. The temperature parameter $\tau > 0$ controls the smoothness of the surrogate. This surrogate is AUC-consistent and admits regret bound (Theorem~\ref{regret bound}). It theoretically guarantees that optimizing this surrogate objective leads to an improvement in AUC. 

The corresponding policy optimization objective:
\begin{equation}
\label{eq:on-policy-consistent-objective}
J_{\text{logistic}}(\theta)
=
-\mathbb{E}_{o_1, o_2 \sim \mathcal{D}}
\Big[
\log\big(1 + \exp(-t/\tau)\big)
\Big],
\end{equation}
where
\begin{equation}
t = \big( lpm_\theta(o_1) - lpm_\theta(o_2) \big)\, \big( R_1 - R_2 \big).
\end{equation}

 We rewrite the above objective as an expectation over a set of responses and take its gradient:
\begin{equation}
\label{CAPO-gradient}
\resizebox{\columnwidth}{!}{
$
\nabla_\theta J_{\text{logistic}}(\theta)=\mathbb{E}_{o_{1:G}\sim \mathcal{D}}
\left[
\frac{1}{G}
\sum_{i=1}^{G}
\tilde{A_i}
\nabla_\theta lpm_\theta(o_i)
\right],
$}
\end{equation}
\begin{equation}
\label{CAPO-advantage}
\resizebox{\columnwidth}{!}{
$
\tilde{A}_i
 =
\begin{cases}
\displaystyle
-\sum_{j: R_j=0} \phi'\!\big( lpm_\theta(o_i) - lpm_\theta(o_j) \big),
& R_i=1, \\[6pt]
\displaystyle
\sum_{j: R_j=1} \phi'\!\big( lpm_\theta(o_j) - lpm_\theta(o_i) \big),
& R_i=0,
\end{cases}
$}
\end{equation}
where $\phi'(t)=-\sigma(-t)$, and $\sigma(\cdot)$ denotes the sigmoid function. 

\paragraph{Denoising via Reference-Model-Based Masking.}
Outcome-based binary rewards fail to assess intermediate reasoning. Consequently, they risk penalizing near-correct logic due to minor errors, while simultaneously reinforcing spurious reasoning that yields correct answers by chance. Such gradient noise can destabilize training or even induce model collapse. 

We leverage the inherent calibration of the reference model (i.e., base model) to assess the quality of the reasoning process: high reference-model perplexity (PPL) typically indicates syntactic or logical flaws, whereas low PPL suggests coherence. This approach is that it requires no additional training overhead while significantly improving training stability.

Based on this, we propose a masking strategy to filter noisy signals. Specifically, we exclude correct responses with reference-model PPL exceeding a threshold \texttt{ref-high} (likely lucky guesses) and incorrect responses with reference-model PPL below \texttt{ref-low} (likely near-correct reasoning penalized as failure). These thresholds are robust and simple to set, empirically corresponding to the upper and lower quartiles of the reference model’s PPL distribution.
\paragraph{Final Objective.}
By replacing the advantage estimation in the GRPO objective (\ref{grpo-obj}) with that in Equation (\ref{CAPO-advantage}), and applying noise mask, we obtain the objective under the behavior policy $\pi_{old}$:
\begin{equation}
\resizebox{\columnwidth}{!}{
$
\begin{aligned}
&
J_{CAPO}(\theta)
\\
&
=\mathbb{E}_{(q,a)\sim \mathcal{D},\, \{o_i\}_{i=1}^{G} \sim \pi_{\theta_{\mathrm{old}}}(\cdot \mid q)}
\Bigg[
\frac{1}{G}\sum_{i=1}^G \frac{1}{|o_i|}\sum_{t=1}^{|o_i|}\min
\\
&
\Big(
r_{i,t}(\theta)\hat{A}_{i}^{CAPO},\mathrm{clip}\big(r_{i,t}(\theta),1-\epsilon,1+\epsilon\big)\hat{A}_{i}^{CAPO}
\Big)\Bigg],
\end{aligned}
$}
\end{equation}
where
\begin{equation}
\hat{A}_{i}^{CAPO}
 =m(o_i)\tilde{A}_{i}~,
\end{equation}
\begin{equation}
\resizebox{\columnwidth}{!}{
$
m(o) =
\begin{cases}
\mathbb{I}\!\left[\mathrm{PPL}_{\text{ref}}(o) \le \texttt{ref-high}\right], 
& \text{if } R(o)=1,\\[6pt]
\mathbb{I}\!\left[\mathrm{PPL}_{\text{ref}}(o) \ge \texttt{ref-low}\right],
& \text{if } R(o)=0.
\end{cases}
$}
\end{equation}
This yields the optimization objective of CAPO. While the PPO-style formulation involves first-order approximation, the core logistic surrogate provides gradients aligned with relative calibration, offering strong theoretical justification. Subsection~4.2 and experiment results further corroborate this alignment.
\subsection{Gradient Analysis }
\begin{figure}
  \centering
  \includegraphics[width=0.8\columnwidth]{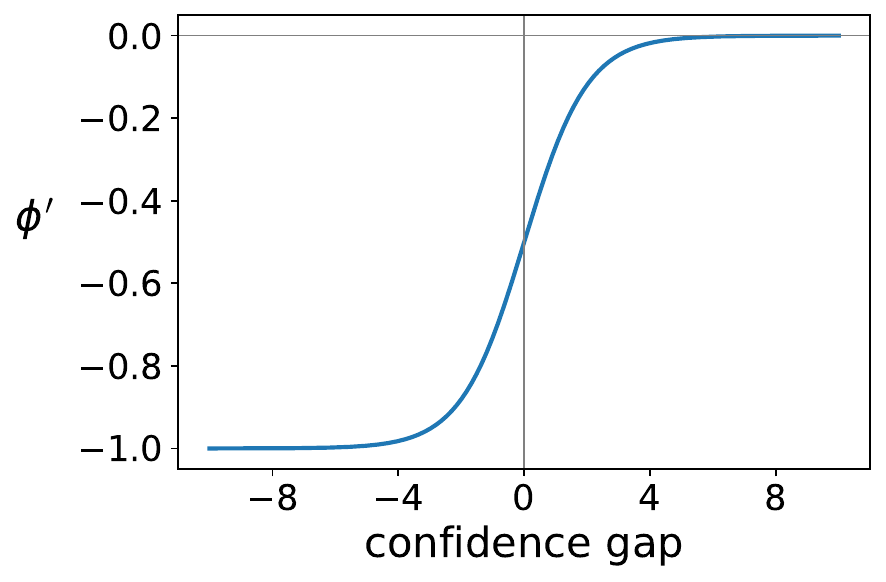}
  \caption{The relationship between advantage and the confidence gap of a correct--wrong response pair.}
  \label{advantage}
\end{figure}
This subsection conducts a comparative analysis on the gradients between GRPO and our proposed objective function.
The gradient induced by our objective is given by Equation (\ref{CAPO-gradient}) (mask term obviated for brevity). $\phi'(t)$ is illustrated in Figure~\ref{advantage}. As can be seen, its magnitude decreases as the confidence gap between correct and incorrect responses increases. This indicates that the gradient places greater emphasis on correct samples with relatively high PPL and incorrect samples with low PPL, while suppressing the influence of samples that are already confidently and correctly ranked.

Samples near the misranking boundary in terms of reference-model PPL represent instances where the model’s confidence estimates are inaccurate; they also represent highly informative samples in the model's decision space.
Consequently, these are the key samples for improving both accuracy and AUC. Our method appropriately amplifies their gradients, while the masking mechanism removes update instability caused by extreme noisy samples. In contrast, GRPO assigns the same reward-based advantage to all positive and negative samples within a group, without accounting for their uncertainty or noise.
\section{Experiment}
\begin{figure*}[t]
    \centering
    \includegraphics[width=\textwidth]{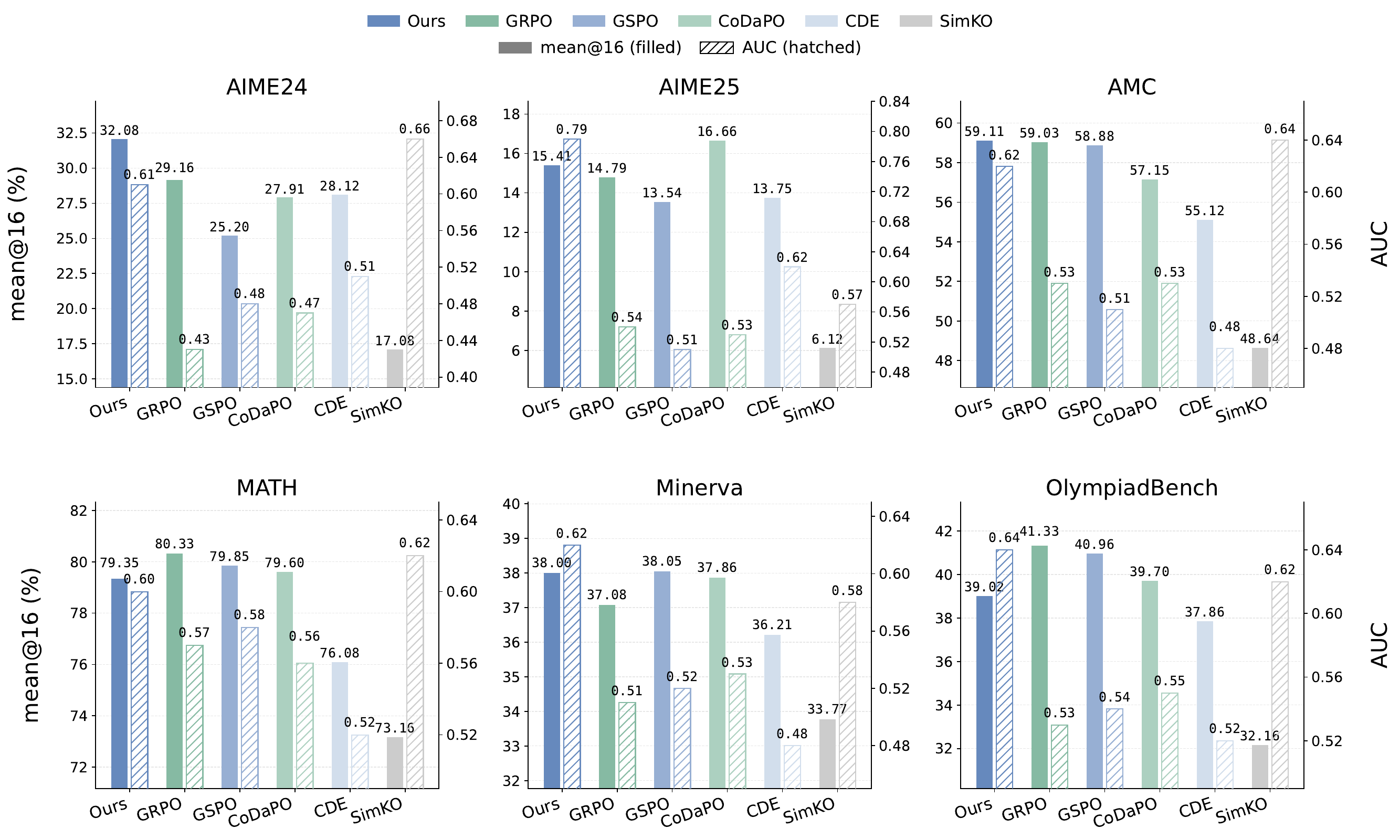}
    \caption{Results of calibration (measured by AUC-mean) and accuracy (measured by mean@16)
    for our method and all baselines on Qwen2.5-Math-7B across six test benchmarks.}
    \label{main_result-7b}
\end{figure*}
\begin{table*}[t]
\centering
\vspace{0.5em}
\begin{tabular}{l|cccccc}
\toprule
\textbf{Model} 
& \textbf{Ours} 
& \textbf{GRPO} 
& \textbf{GSPO} 
& \textbf{CoDaPO} 
& \textbf{CDE} 
& \textbf{SimKO} \\
\midrule
Qwen2.5-Math-1.5B
& \textbf{25.33} 
& 20.33 
& 20.00 
& 21.67 
& 16.67 
& 11.67 \\
Qwen2.5-Math-7B
& \textbf{38.33} 
& 33.33 
& 32.21 
& 31.66 
& 31.66 
& 23.33 \\
\bottomrule
\end{tabular}

\caption{Inference-time scaling accuracy (\%) of different methods on the AIME 2024 and AIME 2025 datasets.}
\label{tab:inference_scaling_acc}
\end{table*}

\subsection{Setting}
\paragraph{Models and Datasets.}
We use Qwen2.5-Math-1.5B and Qwen2.5-Math-7B \citep{yang2024qwen25mathtechnicalreportmathematical} as the base models. The training and validation sets are constructed by randomly sampling 20k and 240 instances, respectively, from the DeepScaler dataset \citep{luo2025deepscaler}. We evaluate all models on six benchmark datasets: AIME~2024, AIME~2025, MATH~500 (\citealp{lightman2023let};\citealp{hendrycks2021measuring}), AMC~2023, Minerva \citep{lewkowycz2022solving}, and OlympiadBench \citep{he2024olympiadbench}.

\paragraph{Baselines and Methods.}
The baselines we compare against fall into two categories: 1) methods designed to improve reasoning capability, including GRPO \citep{shao2024deepseekmath} and GSPO \citep{zheng2025groupsequencepolicyoptimization}. 2) methods proposed to address the calibration issues of GRPO, namely CDE \citep{dai2025cde}, CoDaPO \citep{zhou2025codapo}, and SimKO \citep{peng2025simko}. For our method, $\tau$ is set to 0.6 for the 1.5B model, and 0.5 for the 7B.
Hyperparameter details are provided in Appendix~\ref{experiment setup}.

\paragraph{Evaluation Metrics.}
Following \citet{yue2025does}, we report mean@8 accuracy on the validation set and mean@16 accuracy on the test set.
Model calibration is evaluated by AUC \citep{vashurin2025uncertainty} and Hallucination mitigation by the Precision-Coverage curve.

\begin{figure*}[t]
  \centering
  \includegraphics[width=0.45\linewidth]{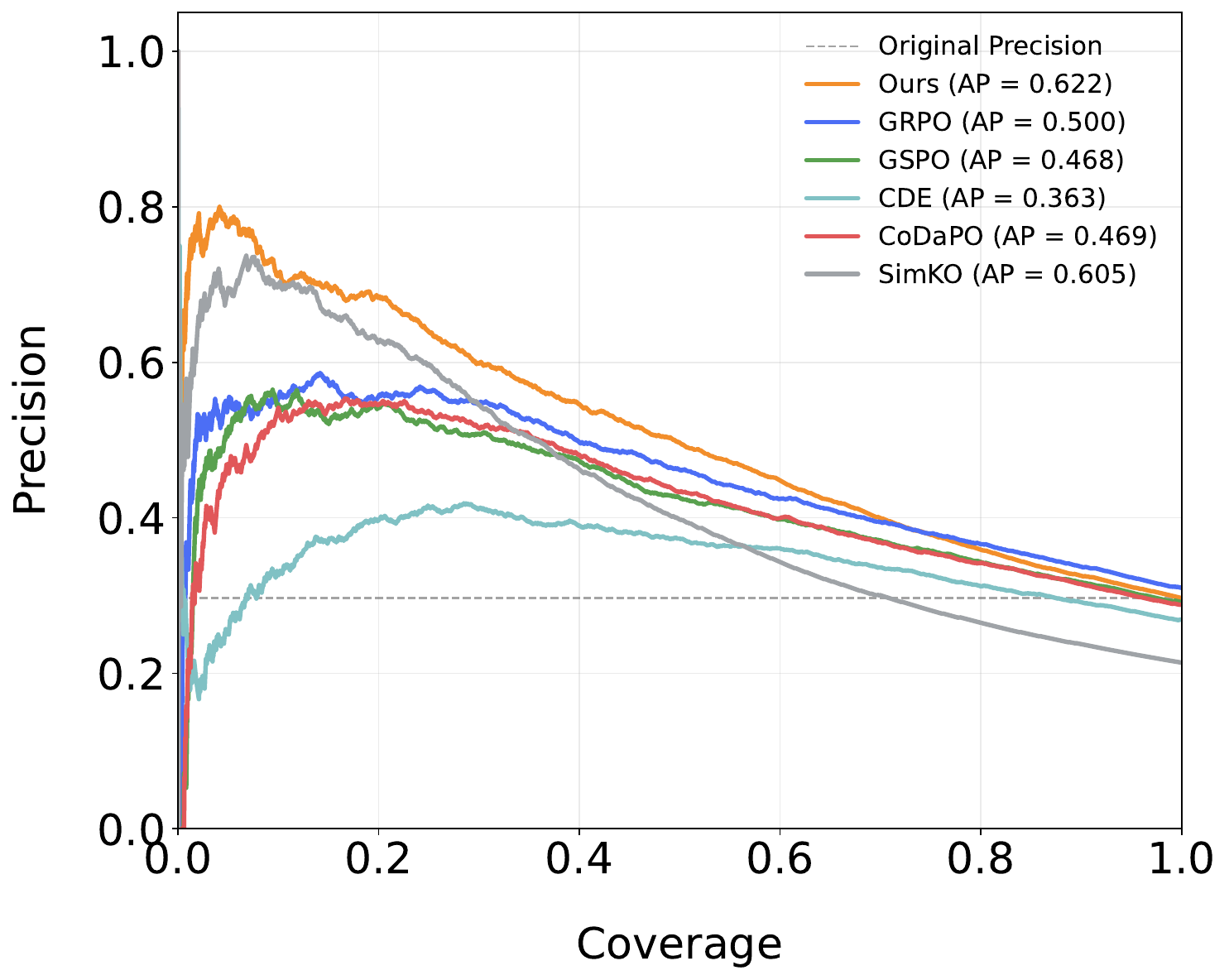} 
  \hspace{0.04\linewidth}
  \includegraphics[width=0.45\linewidth]{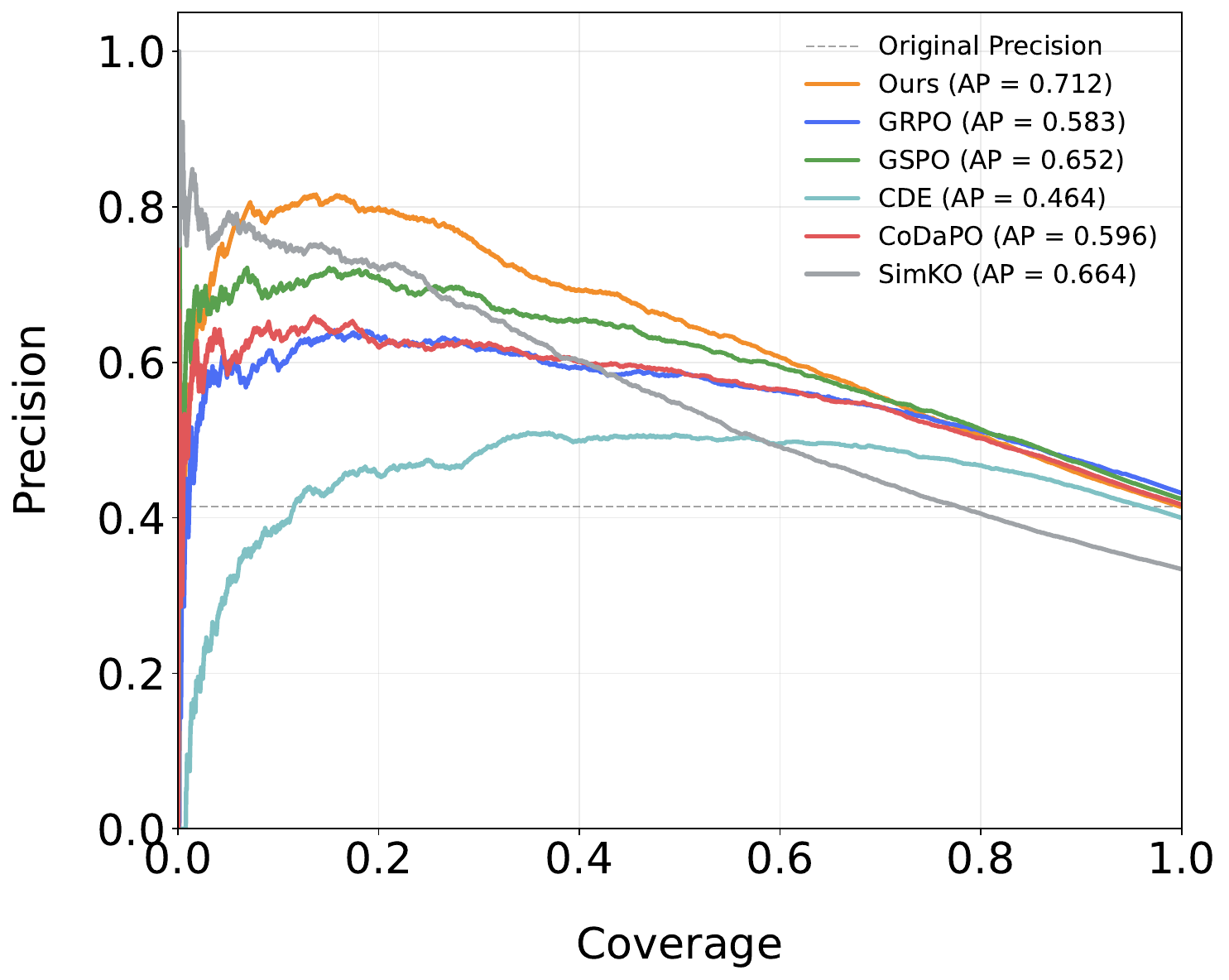}
  \caption {Precision-Coverage curves of our method and all baselines on six test benchmarks for the Qwen2.5-Math-1.5B (left) and Qwen2.5-Math-7B models (right).}
  \label{pc_curve}
\end{figure*}

\subsection{Main Results}
\paragraph{Improved Calibration with Preserved Accuracy}
Extensive experiments demonstrate that our method consistently improves calibration (AUC) across all benchmarks without sacrificing accuracy. As shown in Figure~\ref{main_result-7b} and Appendix~\ref{additional experiment results} Figure~\ref{main_result-1.5b}, compared to GRPO and GSPO, our method significantly improves calibration on all test datasets. In particular, on AIME 2025, CAPO increases AUC by approximately 15\% for the 1.5B model (from 0.63 to 0.78) and by about 25\% for the 7B model (from 0.54 to 0.79). Meanwhile, the mean@8 accuracy achieved by our method is comparable to or even better than that of GRPO, and it attains the highest accuracy on AIME 2024, AIME 2025, and Minerva.

In contrast, other methods aimed at addressing the calibration issue of GRPO provide only limited improvements in AUC and, in some cases, even slightly degrade it on certain datasets. Moreover, none of these methods are able to preserve consistent high accuracy. In particular, SIMKO severely damages accuracy, exhibiting an approximately 12\% drop on AMC and a 7.7\% drop on AIME 2024 compared to GRPO. 

\paragraph{Stable and Consistent Optimization Dynamics}
We further present training dynamics to demonstrate the stability and consistency of our method. Figure~\ref{parento-auc-grpo-gspo}(b)(c) shows AUC-mean on AIME 2024 and AIME 2025 throughout training. While GRPO and GSPO exhibit a gradual degradation in calibration as optimization proceeds, our method steadily improves AUC over training steps, demonstrating sustained calibration optimization rather than checkpoint-specific effects. Figure~\ref{validation-acc} in Appendix~\ref{additional experiment results} reports validation mean@8 accuracy, where only our method maintains stable performance comparable to GRPO for both 1.5B and 7B models. Together, these results indicate that our approach achieves more stable and reliable optimization behavior throughout training.

\paragraph{Hallucination Mitigation.}
We evaluate hallucination mitigation through the precision-coverage trade-off, where a model abstains from answering if its confidence falls below a threshold. Varying this threshold traces a precision-coverage curve that characterizes the fraction of answered questions (coverage) and the accuracy among answered ones (precision). An ideal model should maintain high precision and coverage level. As shown in Figure~\ref{pc_curve}, our method achieves a Pareto-optimal trade-off, consistently dominating GRPO, GSPO, and other baselines. Notably, while SimKO shows high precision at low coverage due to its ranking capability, its overall lower accuracy causes a sharp precision drop as coverage increases. Conversely, although GRPO matches our accuracy, its inferior calibration leads to consistently lower precision. By simultaneously enhancing calibration and accuracy, our method achieves a more favorable precision–coverage trade-off, enabling reliable hallucination mitigation with minimal abstention.

\paragraph{Impact of Calibration on Inference-Time Scaling}
The effectiveness of confidence-based inference-time scaling depends jointly on a model’s accuracy and its calibration quality. We adopt the inference-time scaling algorithm proposed in \citet{zhou2025bridging} (see Appendix~\ref{tts detail} for implementation details). As shown in Table~\ref{tab:inference_scaling_acc}, our method achieves the highest accuracy under inference-time scaling (5\% over GRPO for 1.5B and 7B model), clearly demonstrating the critical role of calibration in inference-time scaling and the effectiveness of our approach.

\subsection{Ablation}
\paragraph{Sensitivity of Hyperparameters}
As detailed in Appendix~\ref{additional experiment results} Figure~\ref{tau}, we conduct sensitivity tests on $\tau \in \{0.4, 0.6, 1.0\}$ and different masking intervals. In both cases, the performance impact is marginal, confirming that our method is robust to the specific selection of $\tau$ and the \texttt{ref-high}/\texttt{ref-low} thresholds.

\paragraph{Ablation of the Denoising Strategy}
We further compare the performance of our method with and without the masking mechanism. As shown in Appendix~\ref{additional experiment results} Figure~\ref{ref}, removing the masking mechanism causes the model entropy to gradually increase, leading to early stagnation or even degradation in accuracy. In contrast, with the masking mechanism enabled, the model entropy remains stable and accuracy improves steadily. These observations demonstrate that noisy samples can severely disrupt optimization stability, and highlight the effectiveness of the proposed masking mechanism. Furthermore, applying only masking to GRPO fails to improve AUC, demonstrating that a consistent surrogate objective is essential.

\section{Conclusion}
In this paper, we have quantitatively analyzed and mitigated the calibration degradation induced by GRPO-style algorithms. We show that this degradation stems from reward-only advantage estimation. This estimation causes the gradient unaligned with relative calibration. To address this issue, we propose CAPO, which incorporates uncertainty-aware advantage estimation grounded in a consistent logistic AUC surrogate, together with a denoising mechanism to ensure stable training. Extensive empirical results demonstrate that CAPO trains models that achieve both high accuracy and calibration, leading to superior performance in hallucination mitigation and inference-time scaling.

\section*{Limitations}
This study mainly evaluates the proposed method on mathematical reasoning benchmarks, which are also the primary focus of GRPO-style Reinforcement Learning with Verifiable Rewards (RLVR) methods. However our method is designed to be broadly applicable to training reasoning models. Its effectiveness on other reasoning tasks such as logical puzzles, commonsense reasoning, and open-domain question answering, remains to be further validated.

\section*{Acknowledgment}
This work is supported by National Key R\&D Program of China (2025ZD0122003), and in part by the Beijing Major Science and Technology Project under Contract no. Z251100008425009.
\bibliography{custom}

\appendix

\section{Additional Proof}
\subsection{Derivation of the Gradient of the GRPO Objective}
\label{gradient}
\paragraph{Proof idea.}
We start from the full PPO clipped surrogate with importance sampling (IS) ratio. As proved by previous works \citep{schulman2017proximalpolicyoptimizationalgorithms, schulman2015trust}, IS ratio ensures the surrogate uses tokens drawn from the old policy while producing
a unbiased gradient estimate for the current policy (given a fixed state).
The clipping further enforces a trust-region-style constraint, and under a first-order
approximation (small policy update), the clipped surrogate has the same leading-order
gradient direction as the original optimization objective. Hence, PPO's gradient matches the
REINFORCE-style gradient. GRPO differs from PPO only by replacing the advantage
estimator with a group-wise baseline $\hat A_i = R_i - \bar R$, so substituting this
advantage into the REINFORCE form yields the corresponding GRPO gradient.

\paragraph{Proof}
Let $r_t(\theta) = \frac{\pi_\theta(a_t\mid s_t)}{\pi_{\theta_{\mathrm{old}}}(a_t\mid s_t)}$
denote the importance sampling ratio.
The PPO clipped objective is
\begin{equation}
\label{eq:ppo-clip-obj}
\begin{aligned}
L^{\mathrm{PPO}}(\theta)
&=
\mathbb{E}_{(s_t,a_t)\sim \pi_{\theta_{\mathrm{old}}}}
\Big[
\min\!\Big(
r_t(\theta)A_t,\\
&\qquad
\mathrm{clip}(r_t(\theta),1-\epsilon,1+\epsilon)\,A_t
\Big)
\Big].
\end{aligned}
\end{equation}
To first order, the gradient of the above off-policy objective reduces to the on-policy REINFORCE-form gradient:
\begin{equation}
\resizebox{\columnwidth}{!}{$
\label{reinforce gradient}
\nabla_\theta L^{\mathrm{PPO}}(\theta)
=\nabla_\theta\mathbb{E}_{(s_t,a_t)\sim \pi_{\theta}}
\left[
A_t\,\nabla_\theta\log\pi_\theta(a_t\mid s_t)
\right],
$}
\end{equation}

The GRPO objective is obtained by estimating the advantage in PPO by group-wise reward differences.
Given a group of $G$ responses $\{o_i\}_{i=1}^G$ with rewards $\{R_i\}$, GRPO uses
\begin{equation}
\tilde{A_i} = R_i - \bar R,
\qquad
\bar R = \frac{1}{G}\sum_{i=1}^G R_i.
\end{equation}
Substituting this $\tilde{A_i}$ into the REINFORCE-form gradient~\ref{reinforce gradient} yields:
\begin{equation}
\resizebox{\columnwidth}{!}{$
\mathbb{E}_{o_{1:G}\sim \mathcal{D}}
\Bigg[
\frac{1}{G}\sum_{i=1}^G \frac{1}{|o_i|}\sum_{t=1}^{|o_i|} (R_i-\bar R)\,\nabla_\theta\pi_\theta(o_{i,t}\mid o_{i,<t})
\Bigg].
$}
\end{equation}
Finally, using the definition
\begin{equation}
lpm_\theta(o_i)
=
\frac{1}{|o_i|}\sum_{t=1}^{|o_i|} \log \pi_\theta(o_{i,t}\mid o_{i,<t}),
\end{equation}
we can obtain Equation~\ref{grpo-grad}.

\subsection{Definition of U-Statistics}
\label{U}

Let $\{Z_1, Z_2, \dots, Z_n\}$ be i.i.d.\ random variables drawn from an underlying distribution $\mathcal{D}$.
Consider a symmetric measurable function (kernel)
\[
h: \mathcal{Z}^k \rightarrow \mathbb{R},
\]
where $k \ge 1$ denotes the order of the kernel and symmetry means that
\[
h(z_1, \dots, z_k)
=
h(z_{\pi(1)}, \dots, z_{\pi(k)})
\]
for any permutation $\pi$ of $\{1,\dots,k\}$.

\paragraph{Definition (U-statistic).}
The \emph{U-statistic} associated with kernel $h$ is defined as
\begin{equation}
U_n
\;\triangleq\;
\binom{n}{k}^{-1}
\sum_{1 \le i_1 < \dots < i_k \le n}
h(Z_{i_1}, \dots, Z_{i_k}).
\label{eq:ustat_def}
\end{equation}
The U-statistic $U_n$ is an estimator of the population quantity
\begin{equation}
\theta
\;\triangleq\;
\mathbb{E}_{(Z_1,\dots,Z_k)\sim\mathcal{D}^k}
\bigl[
h(Z_1,\dots,Z_k)
\bigr].
\label{eq:ustat_target}
\end{equation}

\paragraph{Unbiasedness.}
A key property of U-statistics is that they provide unbiased estimators of the target quantity $\theta$.

\begin{proposition}[Unbiasedness of U-statistics]
\label{prop:ustat_unbiased}
For any symmetric kernel $h$ with finite expectation,
\begin{equation}
\mathbb{E}[U_n] = \theta.
\end{equation}
\end{proposition}

\begin{proof}
By linearity of expectation,
\begin{align}
\mathbb{E}[U_n]
&=
\binom{n}{k}^{-1}
\sum_{1 \le i_1 < \dots < i_k \le n}
\mathbb{E}
\bigl[
h(Z_{i_1}, \dots, Z_{i_k})
\bigr].
\end{align}
Since $\{Z_i\}$ are i.i.d., each summand has expectation $\theta$.
There are exactly $\binom{n}{k}$ such terms, which yields
\begin{equation}
\mathbb{E}[U_n]
=
\binom{n}{k}^{-1}
\cdot
\binom{n}{k}
\cdot
\theta
=
\theta.
\end{equation}
\end{proof}

\paragraph{Relation to the main text.}
In the main text, Equation (\ref{grpo-grad2}) can be viewed as a U-statistic with kernel order $k=2$.
Consequently, it is an unbiased estimator of the expectation in Equation (\ref{grpo-pair-grad}).

\subsection{Completion of the Proof of Theorem~3}
\label{proof of th3}
\paragraph{Inconsistency of $\phi(t)=-t$ for AUC optimization.}
AUC depends only on the \emph{ordering} induced by the scoring function.
In particular, for any $\alpha>0$, scaling the scores does not change AUC:
\begin{equation}
\mathrm{AUC}(f) ~=~ \mathrm{AUC}(\alpha f).
\label{eq:auc-scale-invariant}
\end{equation}
However, the surrogate risk with $\phi(t)=-t$ is \emph{scale-sensitive} and \emph{unbounded}:
\begin{equation}
\mathcal{L}_{-t}(\alpha f)
=
-\mathbb{E}\bigl[\alpha f(Z^+)-\alpha f(Z^-)\bigr]
=
\alpha\,\mathcal{L}_{-t}(f).
\label{eq:surrogate-scales}
\end{equation}
Take any nontrivial scoring function $f$ with $\mathcal{L}_{-t}(f) < 0$ and $AUC(f) < 1$ (not optimal). Consider the sequence
$f_m = \alpha_m f$ with $\alpha_m\uparrow +\infty$.
Then $\mathrm{AUC}(f_m)=\mathrm{AUC}(f)$ for all $m$ by \eqref{eq:auc-scale-invariant},
but $\mathcal{L}_{-t}(f_m)\to -\infty$ by \eqref{eq:surrogate-scales}.
Hence, one can drive the surrogate objective arbitrarily close to its infimum
without improving (or even changing) AUC.
Therefore, minimizing the $\phi(t)=-t$ surrogate does not guarantee approaching
an AUC-optimal scorer, i.e., this surrogate is not consistent for AUC optimization.
\paragraph{Generalizing to Reward-Only Advantage Estimators.}
It is worth noting that this result is not limited to GRPO, but applies to all algorithms that employ reward-only advantage estimators. Since the objectives induced by this class of advantage estimators are always linear with respect to the confidence function, they necessarily satisfy the property in \eqref{eq:surrogate-scales}. As a result, similar scaling counterexamples can be constructed to show that these objectives are not consistent AUC surrogate losses.
\begin{table}[t]
\centering
\begin{tabular}{l c}
\toprule
\textbf{Configuration} & \textbf{Value} \\
\midrule
train\_batch\_size          & 128 \\
ppo\_mini\_batch\_size      & 64 \\
max\_prompt\_length         & 1050 \\
max\_response\_length       & 3046 \\
\texttt{ref-high}        & 2.5 \\
\texttt{ref-low}         & 1.05 \\
$\tau$ (1.5B model)                      & 0.6 \\
$\tau$ (7B model)                      & 0.5 \\
$\epsilon$             & 0.2 \\
learning rate               & $1\times10^{-6}$ \\
entropy\_coeff              & 0 \\
kl\_loss\_coef               & 0 \\
rollout.n                   & 8 \\
validation.rollout.n               & 16 \\
rollout.temperature         & 1.0 \\
validation.temperature     & 1.0 \\
total\_steps(for 1.5B)               & 600 \\
total\_steps(for 7B)               & 400 \\
\bottomrule
\end{tabular}
\caption{Key hyperparameter settings used in our experiments.}
\label{tab:hyperparams}
\end{table}

\subsection{Pseudocode of the Inference-Time Scaling Algorithm}
\label{tts detail}

The inference-time scaling algorithm in Algorithm~\ref{alg:ppl_consistency}
is a direct instantiation of the \emph{Perplexity Consistency} principle
proposed in prior work \citealp{zhou2025bridging}. In our implementation, N is 16.

The core assumption of Perplexity Consistency is that if the model is well-calibrated, responses corresponding to the correct
answer should, on average, exhibit lower perplexity (higher likelihood)
than incorrect ones.

Instead of selecting the single response with minimum perplexity,
the proposed method aggregates confidence at the \emph{answer level}.
All responses that yield the same extracted final answer are grouped,
and their probabilities are summed to form an aggregated confidence score.
This aggregation achieves two desirable properties:

\begin{itemize}
    \item \textbf{Self-consistency}: answers that are repeatedly produced
    by the model with high likelihood are reinforced.
    \item \textbf{Noise robustness}: low-probability or spurious generations
    contribute negligibly to the final decision.
\end{itemize}

Importantly, the entire procedure relies solely on the model's
\emph{endogenous confidence}, quantified by perplexity.
As such, it constitutes a lightweight yet effective form of inference-time
scaling that leverages both the model's calibration quality and reasoning ability.
\begin{algorithm}[t]
\caption{Perplexity-Consistency Based Inference-Time Scaling}
\label{alg:ppl_consistency}
\begin{algorithmic}[1]
\State \textbf{Input:} question $q$, number of samples $N$, model $\pi_\theta$
\State \textbf{Output:} final predicted answer $\hat a$
\State Sample $N$ independent responses $\{y_i\}_{i=1}^N \sim \pi_\theta(\cdot \mid q)$

\For{$i \gets 1$ \textbf{to} $N$}
    \State Extract final answer $a_i$ from response $y_i$
    \State Compute token-level average log-probability:
    \Statex \hspace{1.2em} $\ell_i \gets \frac{1}{|y_i|}\sum_{t=1}^{|y_i|}\log p_\theta\!\big(y_{i,t}\mid y_{i,<t}, q\big)$
\EndFor

\State Group responses by identical extracted answers:
\Statex \hspace{1.2em} $\mathcal{G}(a) \gets \{\, i \mid a_i = a \,\}$

\ForAll{candidate answers $a$}
    \State Compute aggregated confidence:
    \Statex \hspace{1.2em} $C(a) \gets \sum_{i\in\mathcal{G}(a)} \exp(\ell_i)$
\EndFor

\State $\hat a \gets \arg\max_a C(a)$
\State \Return $\hat a$
\end{algorithmic}
\end{algorithm}
\section{Additional Experiment details}
\subsection{Experimental Setup}
Experimental configurations for the 1.5B and 7B models are detailed in Table~\ref{tab:hyperparams}. All methods share these identical settings, with method-specific hyperparameters following their original papers. For our method, we set the temperature parameter $\tau$ to $0.6$ for 1.5B model and 0.5 for 7B model, \texttt{ref-high} and \texttt{ref-low} to the lower and upper quartiles of the reference model’s PPL distribution over correct and incorrect responses, respectively. Experiments were conducted using the verl framework on 8 $\times$ A100 GPUs. Training a single CAPO/GRPO run to convergence takes approximately 24 hours for the 1.5B model and 48 hours for the 7B model. 
\label{experiment setup}
\subsection{Additional Experimental Results}
\label{additional experiment results}

\paragraph{Additional Experimental Results of Qwen2.5-Math-1.5B model and Accuracy Curves}
Figure~\ref{main_result-1.5b} demonstrates the complete experimental results of Qwen2.5-Math-1.5B for all methods over six benchmarks. Our method achieves a significant improvement in calibration over all baselines on the 1.5B model, while maintaining accuracy gains that are comparable to or even surpass those of GRPO. In contrast, other methods either yield only limited improvements in calibration or incur a degradation in accuracy. Figure~\ref{validation-acc} shows the accuracy of all methods on the validation set as a function of training steps. Our method exhibits accuracy improvements comparable to GRPO on both model scales. 

\paragraph{Ablation of Hyperparameters}
As illustrated in Figure~\ref{tau}, our method exhibits low sensitivity to the variation of the hyperparameter $\tau$ across different values ($\tau = 0.4, 0.6, 1.0$). Similarly, when adjusting the values of \textit{ref-high} and \textit{ref-low} to tighten the masking range from $[1.05, 2.5]$ to $[1.25, 2.1]$, Figure~\ref{ref}(b) demonstrates that the performance remains similarly insensitive to these hyperparameter settings.Figure~\ref{mask-ablation} shows that applying the masking mechanism to GRPO alone does not improve calibration, highlighting the importance of the calibration-aware advantage estimator.
\begin{figure*}[t]
    \centering
    \includegraphics[width=\textwidth]{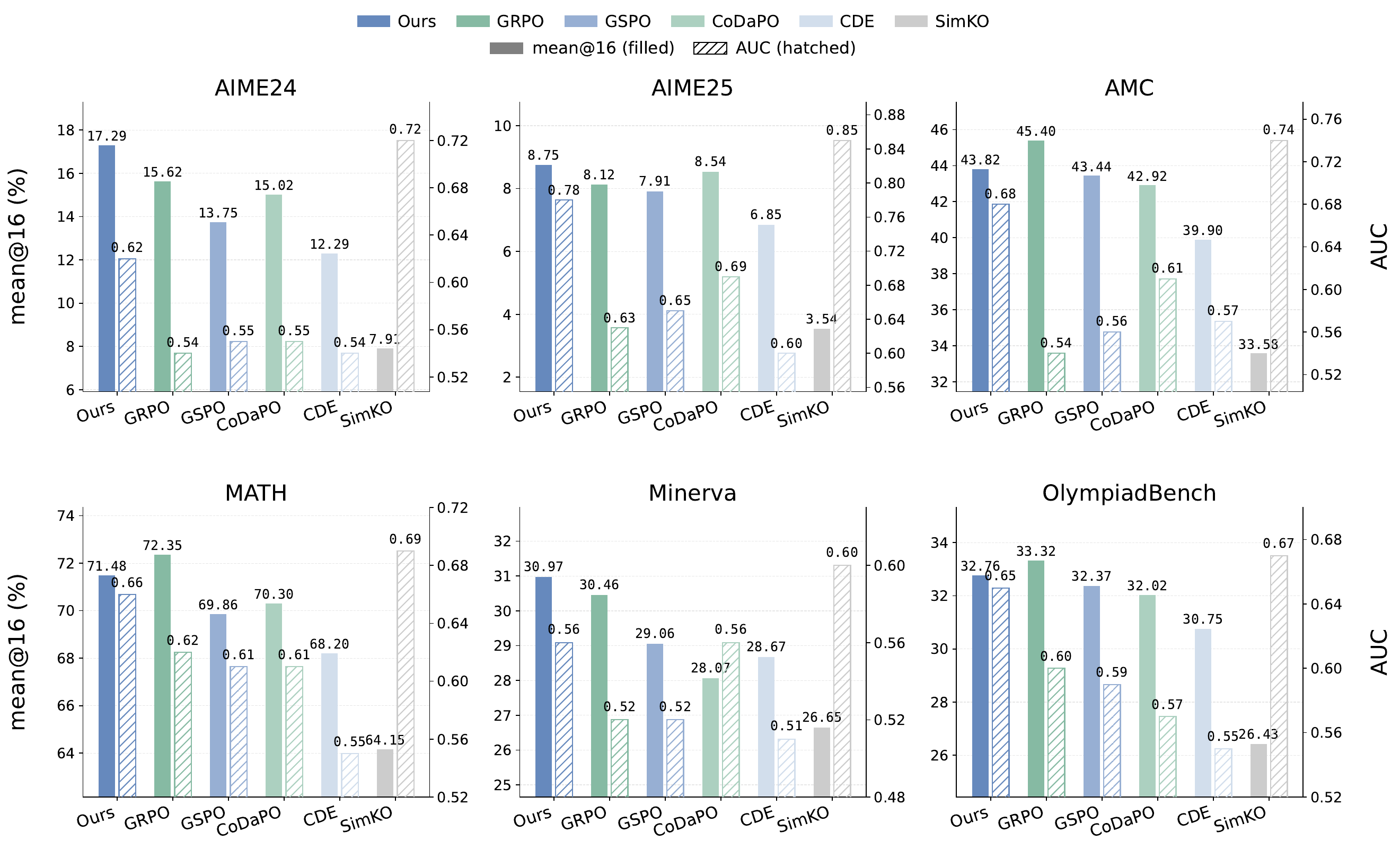}
    \caption{Results of calibration (measured by AUC-mean) and accuracy (measured by mean@16)
    for our method and all baselines on Qwen2.5-Math-1.5B across six test benchmarks.}
    \label{main_result-1.5b}
\end{figure*}
\begin{figure*}[t]
  \centering
  \begin{subfigure}[t]{0.48\linewidth}
    \centering
    \includegraphics[width=\linewidth]{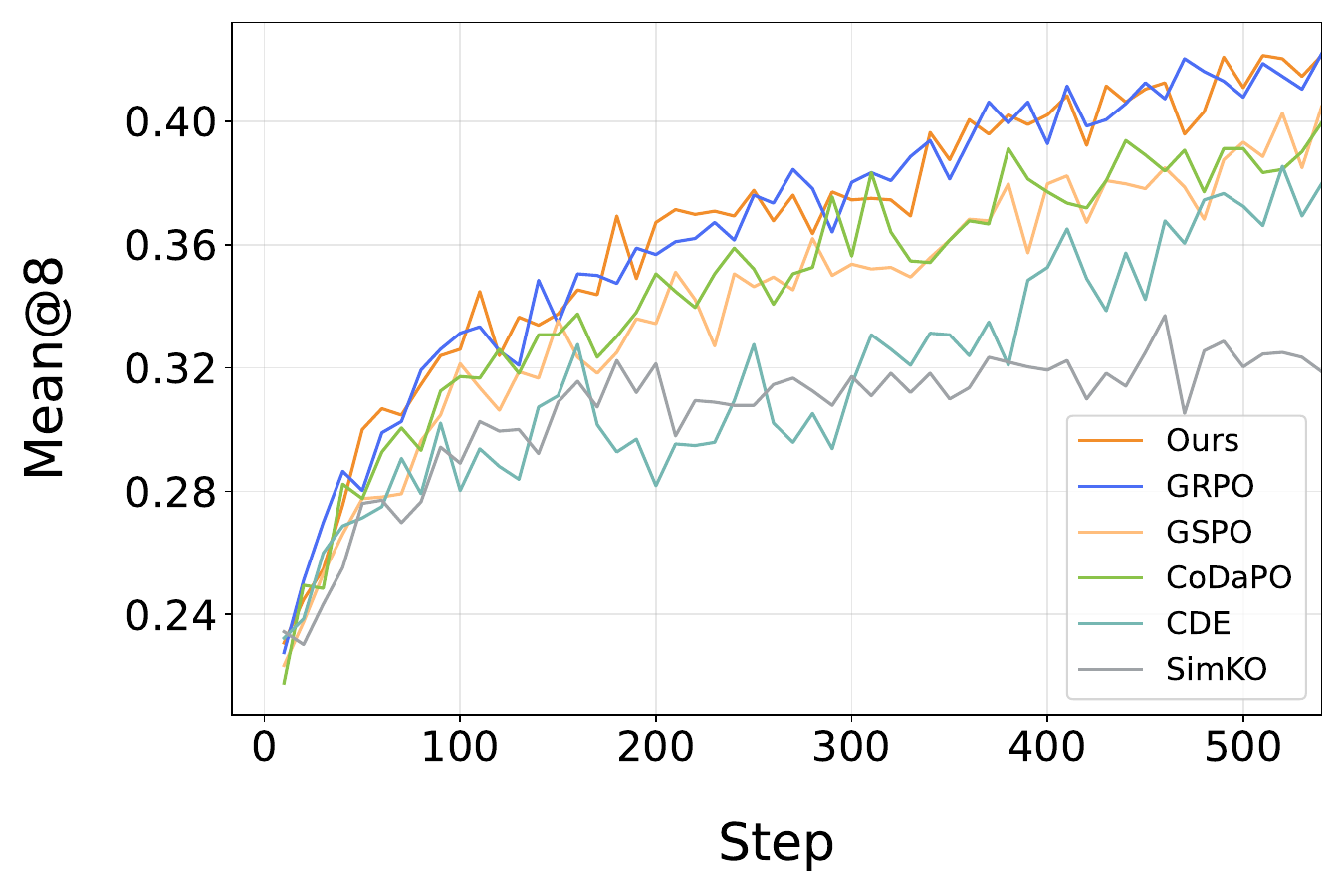}
    \caption{}
    \label{fig:tau-acc}
  \end{subfigure}
  \hfill
  \begin{subfigure}[t]{0.48\linewidth}
    \centering
    \includegraphics[width=\linewidth]{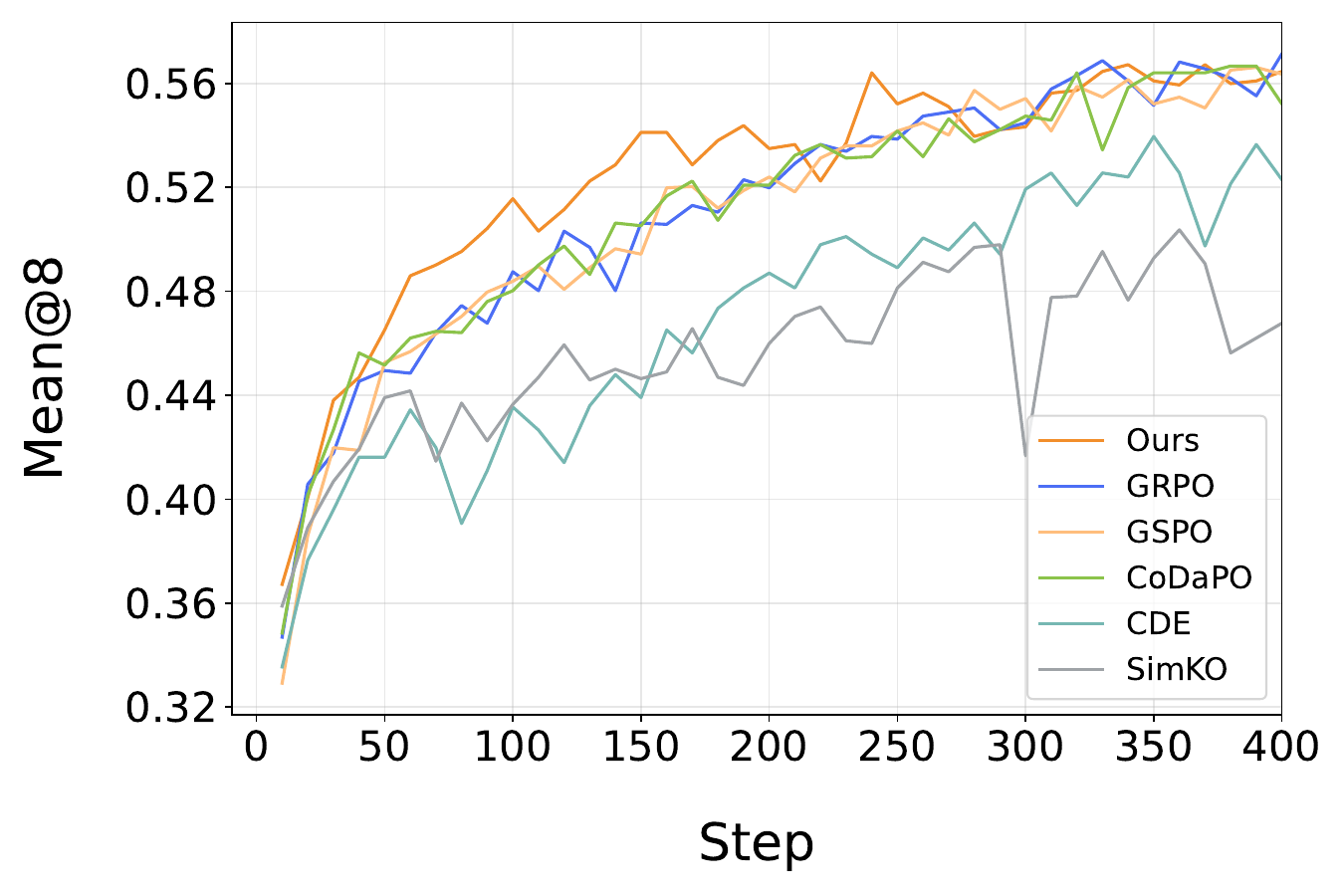}
    \caption{}
    \label{fig:tau-cal}
  \end{subfigure}
  \caption{
  Accuracy trajectories on the validation set over training steps for our method and all baselines on the Qwen2.5-Math-1.5B (left) and Qwen2.5-Math-7B models (right).
  }
  \label{validation-acc}
\end{figure*}
\begin{figure*}[t]
  \centering
  \begin{subfigure}[t]{0.48\linewidth}
    \centering
    \includegraphics[width=\linewidth]{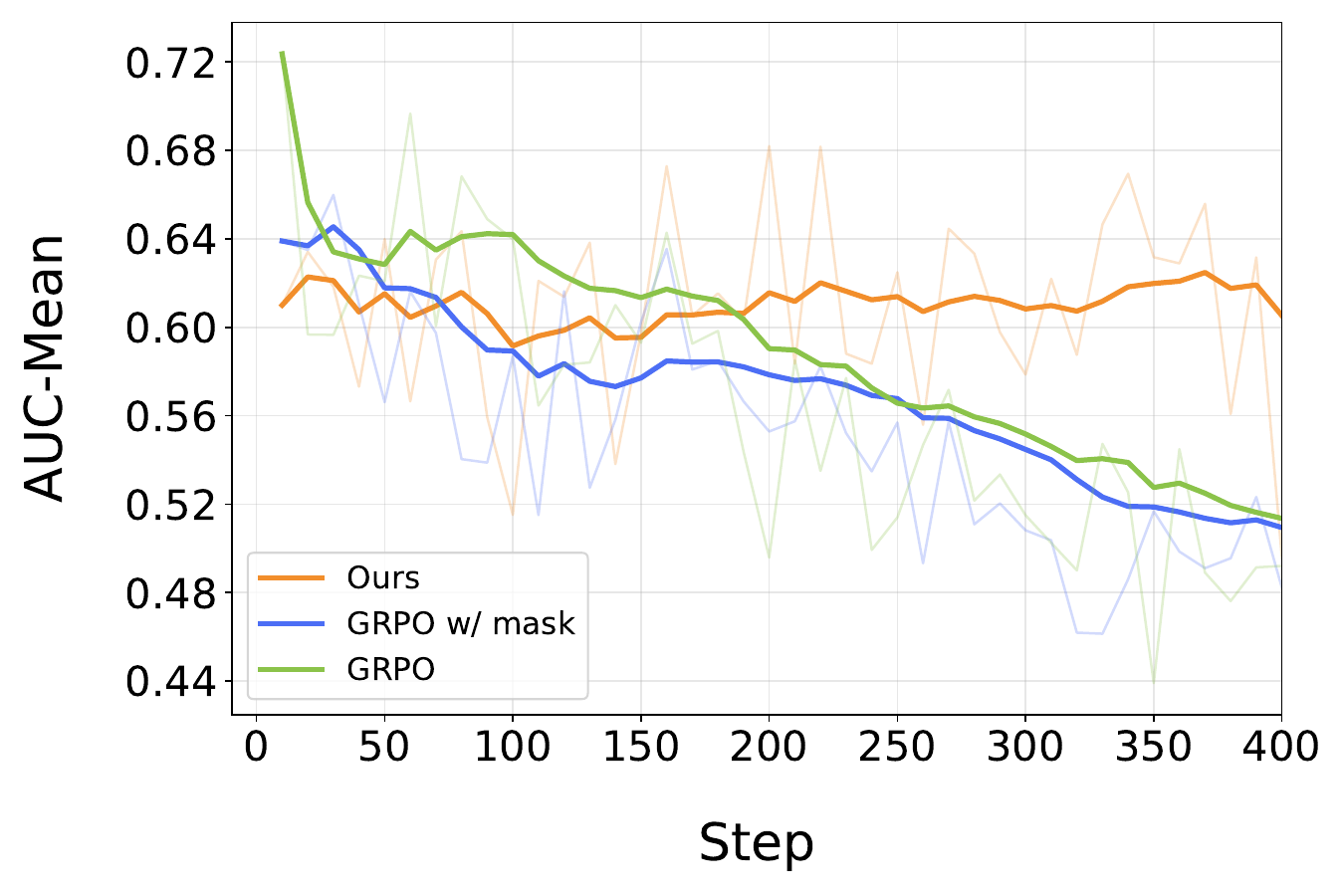}
    \caption{}
    \label{fig:tau-acc}
  \end{subfigure}
  \hfill
  \begin{subfigure}[t]{0.48\linewidth}
    \centering
    \includegraphics[width=\linewidth]{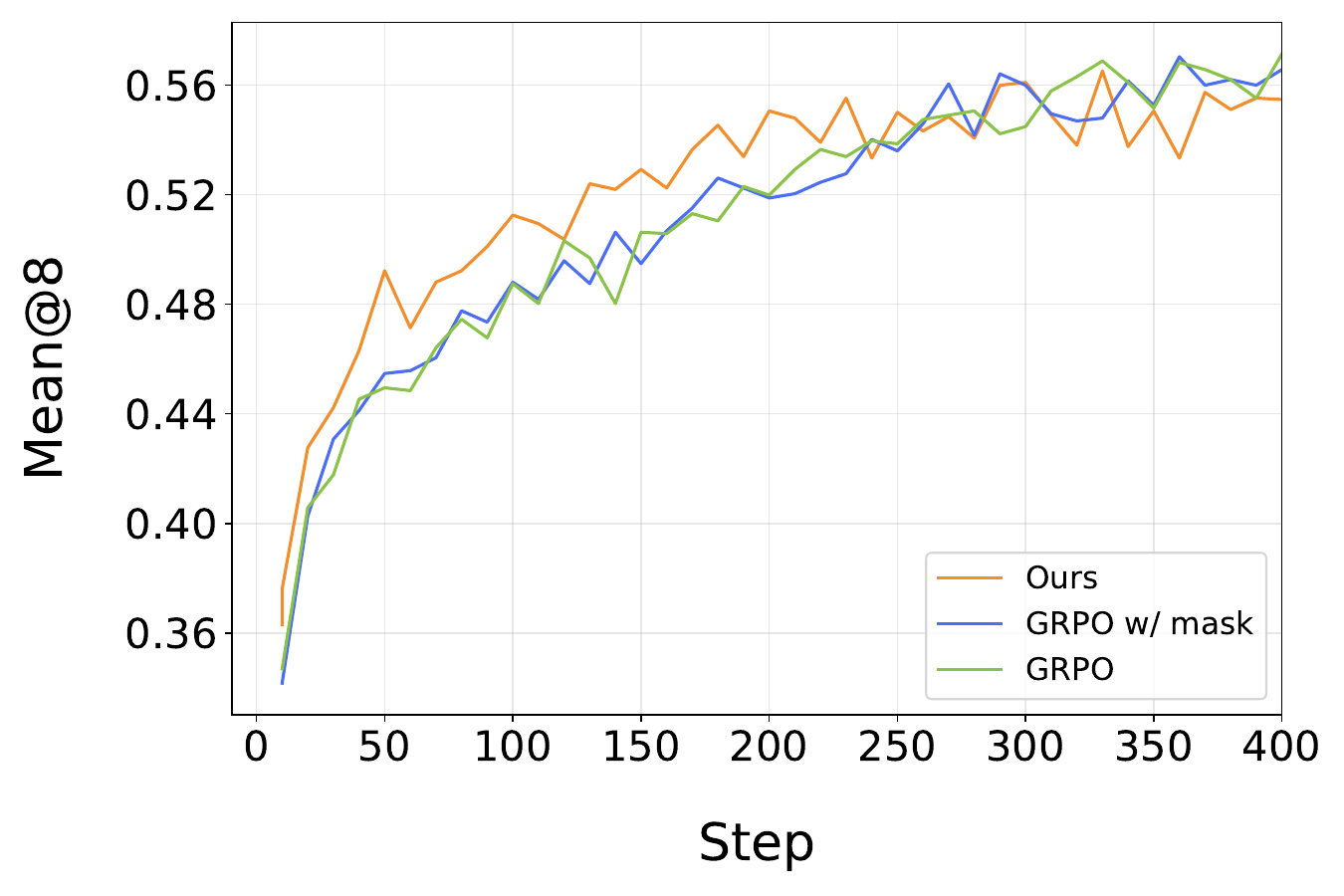}
    \caption{}
    \label{mask-auc}
  \end{subfigure}
  \caption{
  Ablation studies on the effectiveness of applying the masking mechanism alone to GRPO. The results show that the mask by itself does not improve calibration, highlighting the necessity of the calibration-aware advantage estimator.
  }
  \label{mask-ablation}
\end{figure*}

\begin{figure*}[t]
  \centering
  \begin{subfigure}[t]{0.48\linewidth}
    \centering
    \includegraphics[width=\linewidth]{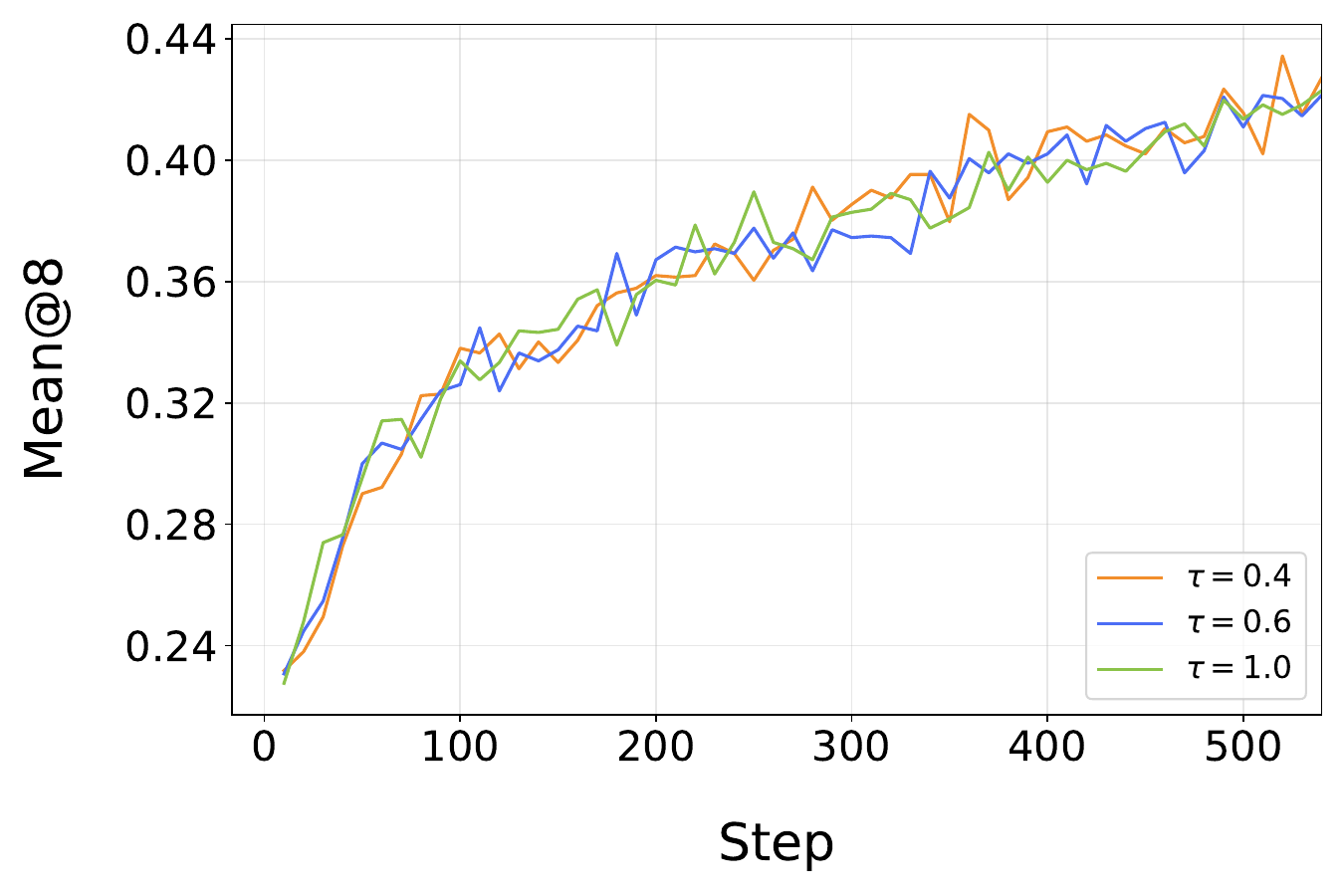}
    \caption{}
    \label{fig:tau-acc}
  \end{subfigure}
  \hfill
  \begin{subfigure}[t]{0.48\linewidth}
    \centering
    \includegraphics[width=\linewidth]{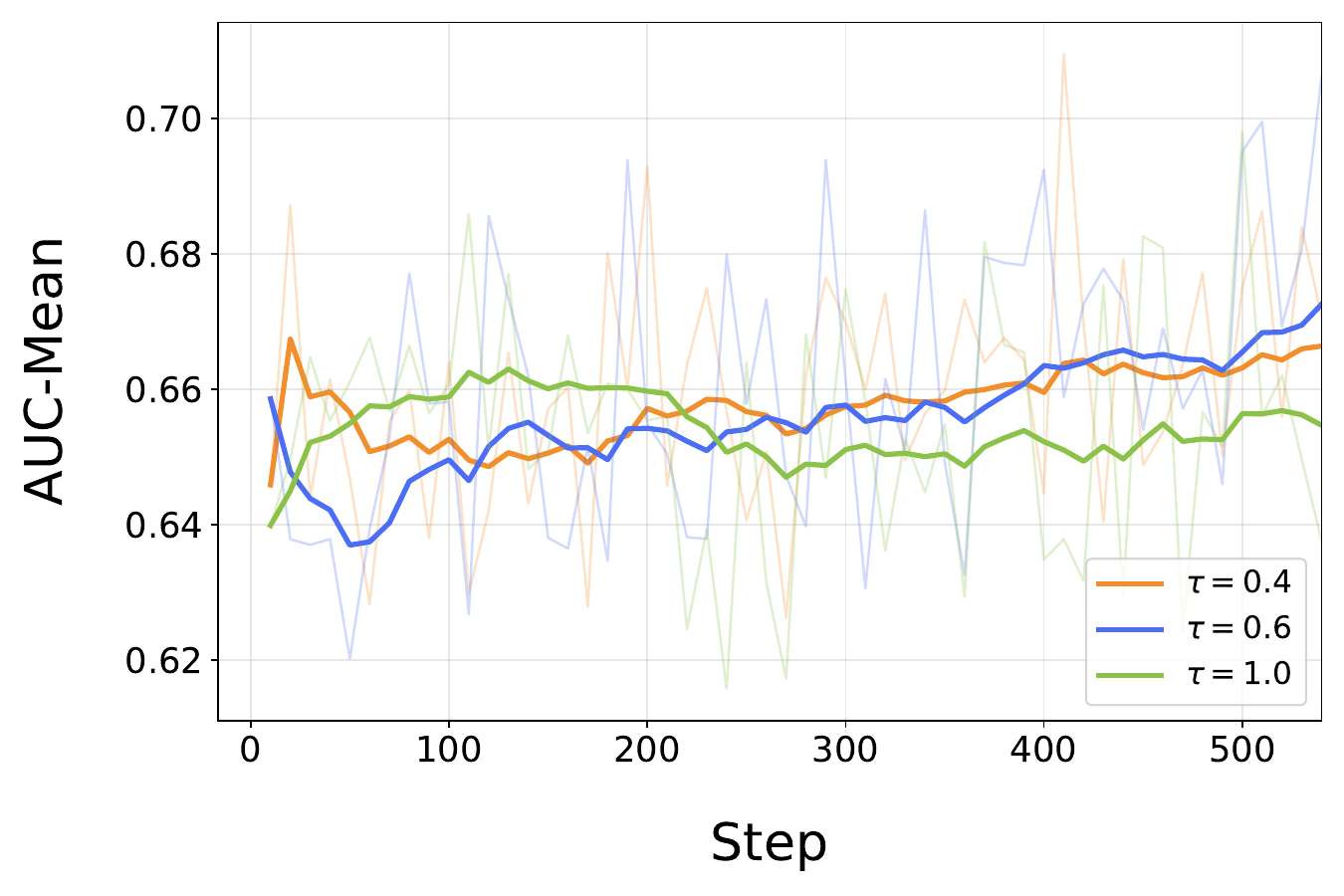}
    \caption{}
    \label{fig:tau-cal}
  \end{subfigure}

  \caption{
  Ablation studies on the sensitivity of accuracy improvement curves (a) and calibration metrics (b) to the hyperparameter $\tau$.
  }
  \label{tau}
\end{figure*}

\paragraph{Ablation of Noise Masking mechanism}
Comparing the performance of the algorithm with and without the masking mechanism, it can be observed from Figure~\ref{ref}(a) that removing the mask leads to a gradual increase in model entropy, causing the accuracy (acc) to stagnate prematurely or even decline. In contrast, with the inclusion of the masking mechanism, the model's entropy remains stable, and the accuracy exhibits a steady improvement. These observations underscore the disruptive impact of noisy data on experimental stability and optimization signals, as well as the effectiveness of the masking mechanism.
\begin{figure*}[t]
  \centering
  \begin{subfigure}[t]{0.48\linewidth}
    \centering
    \includegraphics[width=\linewidth]{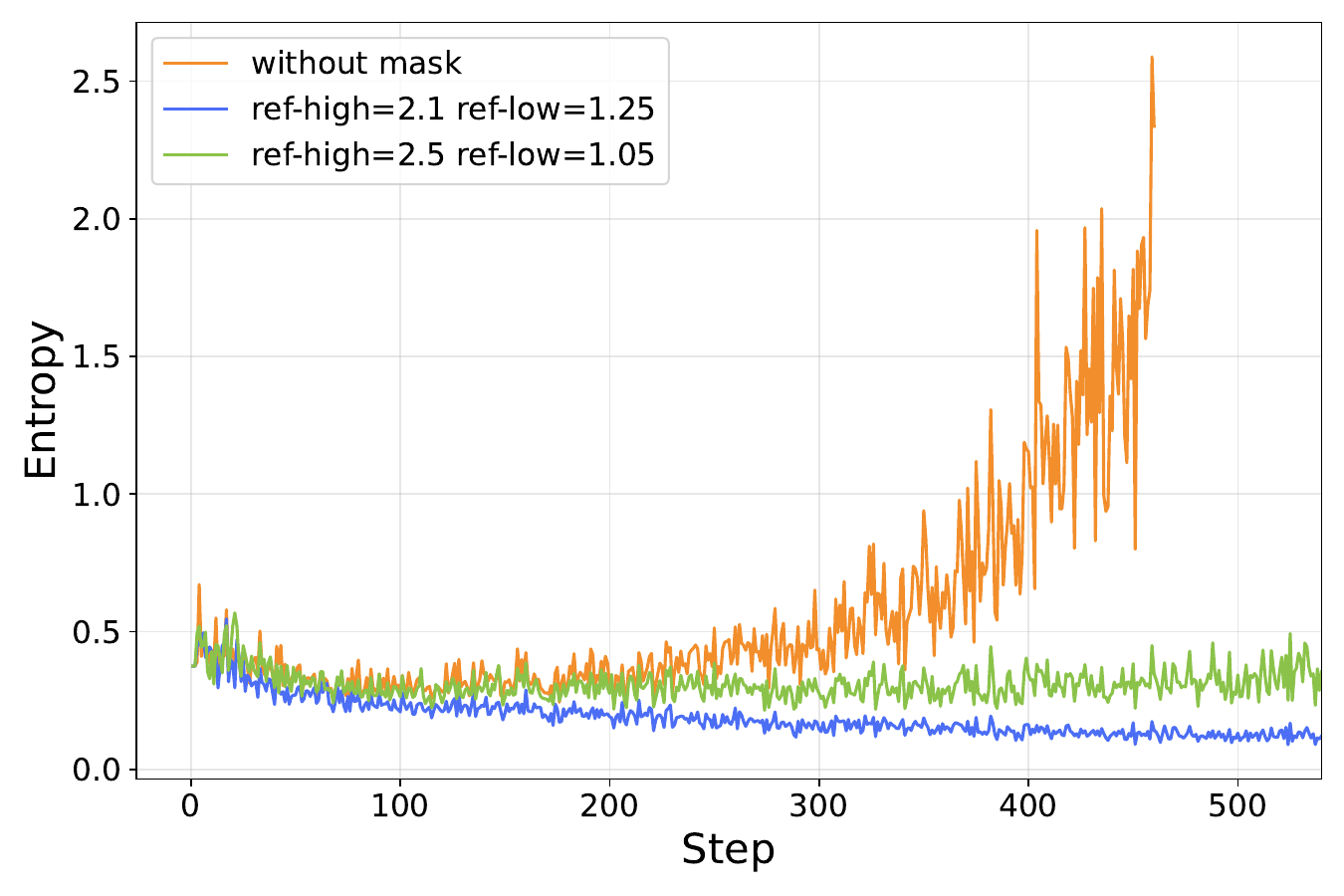}
    \caption{}
    \label{fig:ref-entropy}
  \end{subfigure}
  \hfill
  \begin{subfigure}[t]{0.48\linewidth}
    \centering
    \includegraphics[width=\linewidth]{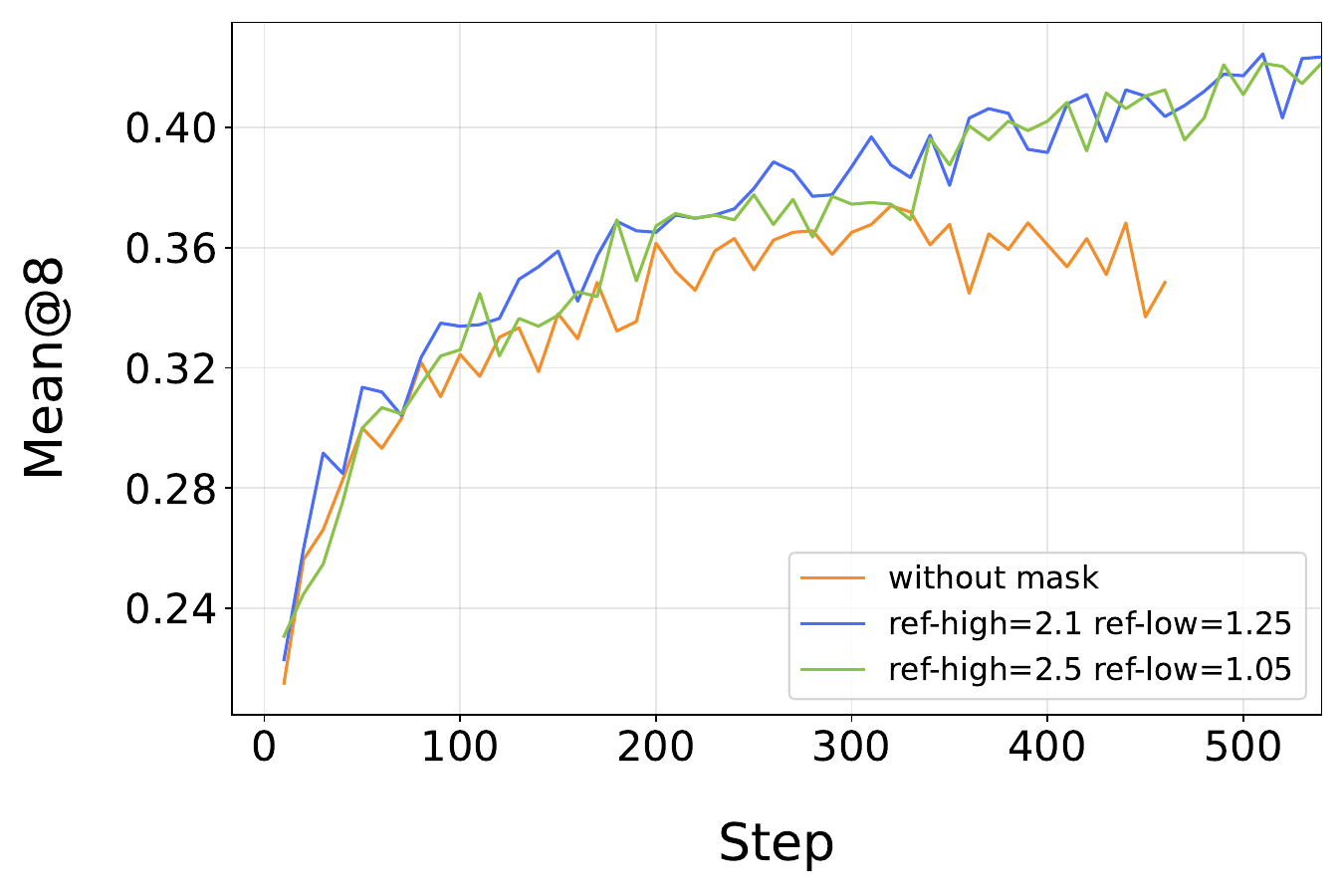}
    \caption{}
    \label{fig:ref-acc}
  \end{subfigure}

  \caption{
  Ablation studies on the impact of the noise-masking mechanism on training stability (a)
  and the sensitivity to the \texttt{ref-high} and \texttt{ref-low} hyperparameters (b).
  }
  \label{ref}
\end{figure*}

\end{document}